\documentclass[letterpaper,journal,compsoc]{IEEEtran}
\usepackage{amsmath,amsfonts}
\usepackage{amsthm,amssymb}
\usepackage{array}
\usepackage{graphicx}
\usepackage[caption=false,font=footnotesize]{subfig}
\usepackage{textcomp}
\usepackage{stfloats}
\usepackage{url}
\usepackage{wrapfig}
\usepackage{verbatim}
\usepackage{xcolor}
\usepackage{lipsum}
\usepackage{tikz}
\usepackage{enumitem}
\usepackage{booktabs}
\usepackage{orcidlink}
\usepackage{hyperref}
\def\BibTeX{{\rm B\kern-.05em{\sc i\kern-.025em b}\kern-.08em
    T\kern-.1667em\lower.7ex\hbox{E}\kern-.125emX}}
\usepackage{balance}
\usepackage{cite}         
\usepackage[ruled,linesnumbered]{algorithm2e}
\newtheorem{thm}{Theorem}
\newtheorem{lem}{Lemma}
\newtheorem{defn}{Definition}

\newtheorem{Prop}{Proposition}

\newcommand{\descr}[1]{\noindent \textbf{#1}}

\begin{document}
\title{Near-Optimal Algorithms for Instance-level Constrained $k$-Center Clustering}

\author{Longkun Guo\textsuperscript{\orcidlink{0000-0003-2891-4253}}, 
Chaoqi Jia\textsuperscript{\orcidlink{0000-0002-6548-390X}},  
Kewen Liao\textsuperscript{\orcidlink{0000-0003-0371-6525}}, 
Zhigang Lu\textsuperscript{\orcidlink{0000-001-5102-6217}},  
and Minhui Xue\textsuperscript{\orcidlink{0000-0002-9172-4252}}

\thanks{Manuscript created 15 June, 2024; revised 06 April 2025; accepted 18 May 2025.  
(\textit{Corresponding author: Longkun Guo.})

    Longkun Guo is with School of Mathematics and Statistics, Fuzhou University, Fuzhou 350116, China  (e-mail:longkun.guo@gmail.com)
    
    Chaoqi Jia is with School of Accounting, Information Systems and Supply Chain, RMIT University, Melbourne, VIC 3000, Australia
    
    Kewen Liao is with School of Information Technology, Deakin University, Burwood VIC 3125, Australia
    
    Zhigang Lu is with Western Sydney University, NSW 2751, Australia
    
    Minhui Xue is with CSIRO's Data61, Sydney 2015, Australia}}

\markboth{IEEE TRANSACTIONS ON NEURAL NETWORKS AND LEARNING SYSTEMS,~Vol.~, No.~, ~2025}{G\MakeLowercase{uo et al.}: Near-Optimal Algorithms for Instance-level Constrained $k$-Center Clustering}

\IEEEpubid{0000--0000/00\$00.00~\copyright~2025 IEEE}

\maketitle
\begin{abstract}
Many practical applications impose a new challenge of utilizing instance-level background knowledge  (e.g., subsets of similar or dissimilar data points) within their input data to improve clustering results. In this work, we build on the widely adopted $k$-center clustering, modeling its input instance-level background knowledge as must-link and cannot-link constraint sets, and formulate the constrained $k$-center problem. Given the long-standing challenge of developing efficient algorithms for constrained clustering problems,  we first derive an efficient approximation algorithm for constrained $k$-center at the best possible approximation ratio of 2 with LP-rounding technology. Recognizing the limitations of LP-rounding algorithms including high runtime complexity and challenges in parallelization, we subsequently develop a greedy algorithm that does not rely on the LP and can be efficiently parallelized. This algorithm also achieves the same approximation ratio 2 but with lower runtime complexity. Lastly, we empirically evaluate our approximation algorithm against baselines on various real datasets, validating our theoretical findings and demonstrating significant advantages of our algorithm in terms of clustering cost, quality, and runtime complexity. We make our code and datasets publicly available to facilitate reproducibility and further research\footnote{\url{https://github.com/ChaoqiJia/TNNLS-constrained_kCenter}}.
\end{abstract}


\section{Introduction}
\label{sec:Intro}

\IEEEPARstart{C}{enter-based} clustering is a fundamental unsupervised learning technique in machine learning that focuses on organizing similar data points into groups or clusters using various distance metrics. Classical clustering problems such as $k$-means~\cite{1_lloyd1982least}, $k$-center~\cite{2_gonzalez1985clustering_minmax, 3_hochbaum1985best_2_k_center}, and $k$-median~\cite{4_charikar1999constant}  are all known for their $\mathcal{NP}$-hard complexity. Among these methods, $k$-center is particularly notable for its robustness against outliers, its ability to provide guaranteed low approximation ratios, and its computational efficiency and scalability. Unlike $k$-means, $k$-center yields deterministic outcomes and aims for clusters that are balanced in terms of their maximum distances to the nearest cluster center, effectively minimizing the largest covering radius across clusters.
Extended from this basic $k$-center objective, there has been significant research interest in devising $k$-center based applications and their associated optimization models. These include specialized variations like capacitated $k$-center~\cite{5_khuller2000capacitated}, which introduces limits on cluster capacities; $k$-center with outliers~\cite{6_charikar2001algorithms}, focusing on handling data anomalies; and advanced models such as minimum coverage $k$-center~\cite{7_lim2005k}, connected $k$-center~\cite{8_ester2006joint}, and recently fair $k$-center~\cite{9_chierichetti2017fair, 10_kleindessner2019fair, 11_bera2022fair} and distributed $k$-center~\cite{12_icmlHuangFH0023}.

In various machine learning applications~\cite{13_zhang2013effective, 14_liu2017private}, the prevalence of unlabeled data and the scarcity of labeled samples — owing to labeling costs — pose significant challenges, especially when the actual ground truth is unavailable or obscured~\cite{15_basu2004active}. Nonetheless, the presence of background knowledge, whose information could be weaker but easier to obtain than labels, indicating whether specific data point pairs should be grouped together or kept separate, was shown to significantly enhance the efficacy of center-based clustering~\cite{16_basu2008constrained}. Leveraging such background knowledge requires more sophisticated clustering models that incorporate must-link (ML) and cannot-link (CL) constraints to reflect this auxiliary information \cite{17_wagstaff2000clustering}. The introduction of clustering with instance-level constraints marked a pivotal development in utilizing background knowledge to guide clustering decisions, necessitating that data points linked by an ML constraint be assigned to the same cluster, while those connected by a CL constraint are distributed into distinct clusters. 

\IEEEpubidadjcol
\subsection{Problem Formulation}
\label{subsec:Problem}

Formally, consider an input dataset $P =\{ p_{1},\dots,p_{n}\}$ within the metric space. The goal of the $k$-center problem is to identify a subset of cluster centers  $C \subseteq P$ to minimize the maximum distance from any point in $P$ to its nearest center in $C$. Furthermore,  denoting the distance function between any two data points as $d(\cdot,\,\cdot)$, we define the distance between a point $p \in P$ and a center set $C$  as $d(p,\,C)=\min_{c\in C}d(p,\,c)$. Thus, the $k$-center problem seeks an optimal set of centers $C^*$ of size at most $k$ that minimizes the maximum distance, or the min-max radius, defined as $r^* = \max_{p \in P} d(p,\,C^*)$. Formally, the objective is to find
\[C^* = \underset{C \subseteq P, |C| \le k}{\arg\min} \, \underset{p\in P} {\max}\,d(p,\,C).\]

Expanding upon the $k$-center problem, we introduce a variant with predefined must-link (ML) and cannot-link (CL) constraints, incorporating background knowledge into the clustering process. Instead of the traditional pairwise constraint approach~\cite{18_wagstaff2001constrained}, we utilize an equivalent set-based formulation for both ML and CL constraints. Specifically, ML constraints are represented as a collection of sets ${\mathcal{X}} =\{X_{1}, \ldots, X_{h}\}$, where each $X_{i} \subseteq P$ is a maximum set of points connected by mutual ML relationships. CL constraints are similarly denoted as ${\mathcal{Y}} = \{Y_{1}, \ldots, Y_{l}\}$, with each $Y_{i} \subseteq P$ and $|Y_{i}| \leq k$, formed from mutual CL relationships. For a point $p$, let $\sigma(p)$ indicate its assigned cluster center. The ML and CL constraints mandate that for every set $X \in \mathcal{X}$, $\forall (p,q) \in X$ if and only if $\sigma(p) = \sigma(q)$, and for every set $Y \in \mathcal{Y}$, $\forall (p,q) \in Y$ if and only if $\sigma(p) \neq \sigma(q)$. By design, sets within ${\mathcal{X}}$ are mutually exclusive, allowing any overlapping ML sets to merge due to the transitive nature of ML relationships.  In contrast, CL sets within  ${\mathcal{Y}}$ may be intersecting or disjoint, where there are applications for either of them.
This paper focuses on developing approximation algorithms specifically for disjoint CL sets for two primary reasons: (1) Disjoint CL sets are of significant interest in numerous applications; (2) The strategies devised for disjoint CL sets can be extended to tackle scenarios where CL sets intersect. This extension is possible through techniques such as merging points that are equivalent across the sets or eliminating points at the intersections, thereby allowing for a broader application of the proposed methods.

\subsection{Challenge}
\label{subsec:Challenge}

The biggest challenge faced when adopting constrained clustering with background knowledge is its computational complexity. As noted, most clustering problems are inherently $\mathcal{NP}$-hard, despite that there usually exist either efficient heuristic algorithms with no performance guarantee or approximation algorithms with non-practical performance ratio and high runtime complexity. Moreover, clustering models with instance-level ML and CL constraints leads to the following severe approximation and computation barriers (particularly caused by the CL constraints) that significantly limit their use:

\begin{Prop}
\label{prop:inapprox}
\cite{19_davidson2007complexity}
It is $\mathcal{NP}$-complete even only to determine whether an instance of the CL-constrained clustering problem is feasible.
\end{Prop}

Arbitrarily intersected CL constraints were known to be problematic to clustering as their inclusion leads to a computationally intractable feasibility problem, as stated in the above theorem. That means, under the assumption of $\mathcal{P} \neq \mathcal{NP}$, it is impossible to devise an efficient polynomial time algorithm to even determine, for an instance of the CL-constrained clustering problem (e.g., CL-constrained $k$-center or $k$-means), whether there exists a clustering solution satisfying all arbitrary CL constraints irrespective of the optimization objective. This inapproximability result can be obtained via a reduction from the $k$-coloring problem and, we believe, has hindered the development of efficient approximation algorithms for constrained clustering despite their many useful applications~\cite{16_basu2008constrained}. For instance, for the closely related constrained $k$-means problem with both ML and CL constraints, only heuristic algorithms~\cite{18_wagstaff2001constrained, 20_davidson2005clustering} without performance guarantee were known. Therefore, a strong motivation for us is to algorithmically overcome the long-standing theoretical barriers on constrained $k$-center that could have been prohibiting its wide adoption (like $k$-center) in practice. We also hope that our proposed techniques can inspire novel solutions to other more intricate problems like constrained $k$-means.

\subsection{Other Related Work}
\label{subsec:Other}

\descr{{Constrained clustering}} Instance-level or pairwise ML and CL constraints have been widely adopted in clustering problems such as $k$-means clustering~\cite{18_wagstaff2001constrained,21_jia2023efficient}, spectral clustering~\cite{22_coleman2008spectral} and hierarchical clustering~\cite{23_davidson2009using}. Basu et al.~\cite{16_basu2008constrained} have collated an extensive list of constrained clustering problems and applications. As confirmed by~\cite{24_xing2002distance,18_wagstaff2001constrained}, instance-level constraints are beneficial for improving the clustering quality.

\descr{{Constrained $k$-center}} Several studies have included $k$-center clustering with instance-level constraints in rather limited settings. Davidson et al.~\cite{25_davidson2010sat} were the pioneers to consider constrained $k$-center when $k$ is 2. They incorporated an SAT-based framework to obtain an approximation scheme with $(1 + \epsilon)$-approximation for this extreme case. Brubach et al.~\cite{26_brubach2021fairness} studied $k$-center only with ML constraints and achieved an approximation ratio $2$. In contrast, we (including our constructed baseline methods) neither consider a limited special case (i.e., with a very small cluster number $k$) nor only the much simpler ML constraints. A comparison on the runtime of the previous algorithms and ours can be found in Tab.~\ref{tab:compar_complexity_thm} in the supplementary material.

\subsection{Results and Contribution}
\label{subsec:Results}

{
In this paper, we gradually develop an efficient solution to constrained $k$-center clustering with performance guarantees in both theory and practice. Moreover, the experiments demonstrate that our method can extend to general constrained $k$-center clustering with intersecting CL sets by leveraging the structure of disjoint CL sets, achieving comparable performance in practice. 
The main contributions of this paper can be summarized as follows:
   
\begin{itemize}[leftmargin=20pt] 
    \item Propose the first constant factor approximation for constrained $k$-center with disjoint CL sets based on a novel structure called \textit{Reverse Dominating Set} (RDS), achieving the best possible provable ratio of 2.
    
 \item Devise two methods for computing RDS based on the construction of an auxiliary graph: the first is based on a designated LP with an integral polyhedron (Lem.~\ref{lem:integeral}); and the second is a greedy algorithm by exploiting the relationship between RDS and a maximum matching of the auxiliary graph, which finds an RDS in runtime $O(k^2)$ (Alg.~\ref{alg:fastgreedy}).

 \item Extensive experiments are carried out on a variety of real-world datasets to demonstrate the clustering effectiveness and efficiency of our proposed approximation algorithm with theoretical guarantees.
   
\end{itemize}}

\section{Algorithm for CL-Constrained \texorpdfstring{$k$}{k}-Center}
\label{sec:Alg}

In this section, we propose a threshold-based algorithm for CL $k$-center and show it deserves an approximation ratio of 2. The key idea of our algorithm is to incrementally expand a set of centers while arguably ensuring that each center is in a distinct cluster of the optimal solution. In the following, we first introduce a structure called \textit{reverse dominating set} (RDS) and propose an algorithm using the structure to grow the desired center set; then, we propose a linear programming (LP) relaxation and use it to find a maximum RDS. Moreover, we accelerate the algorithm by devising a faster LP primal-dual algorithm for finding the maximum RDS.

For briefness, we first assume that the optimal radius ($r^*$) for the constrained $k$-center problem is already known and utilize it as the threshold. Aligns with previous studies on threshold-based algorithms~\cite{27_badanidiyuru2014streaming}, we discuss the problem with both ML and CL constraints without knowing $r^*$ in the next section.
\begin{figure*}[t]
    \captionsetup{justification=centering}

    \subfloat[]{
        \includegraphics[width=0.318\textwidth]{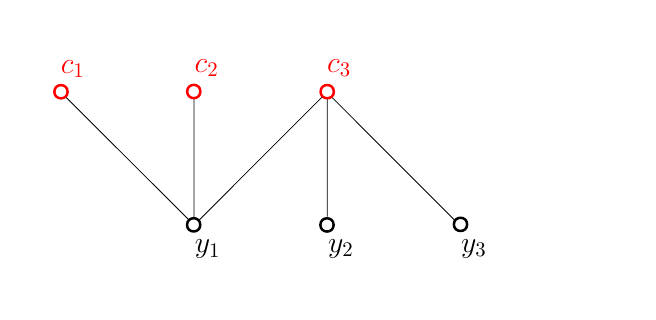}
        \label{subfig:ProcessofRDS_S1}
    }
    \subfloat[]{
        \includegraphics[width=0.318\textwidth]{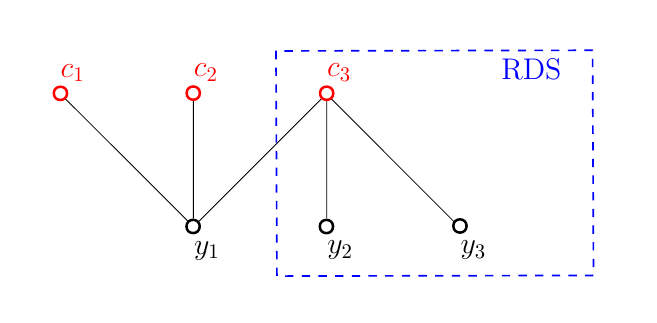}
        \label{subfig:ProcessofRDS_S2}
    }
    \subfloat[]{
        \includegraphics[width=0.318\textwidth]{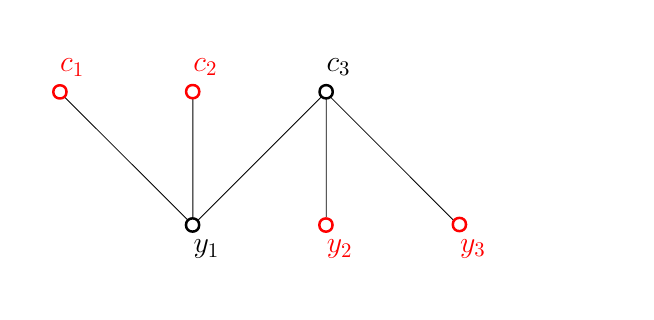}
        \label{subfig:ProcessofRDS_S3}
    }
    \caption{An example of augmenting the current center set by an RDS: (a) The auxiliary graph with the current center set $C=\{c_1,\,c_2,\,c_3\}$ and a CL set $Y=\{y_1,\,y_2,\,y_3\}$; (b) An RDS $(Y', C')$ in which $C' = \{c_3\}$ and $Y'=\{y_{2},y_{3}\}$; (c) The augmented center set (in red).}
    \label{fig:ProcessofRDS}
\end{figure*}

\begin{algorithm}[t]
    \SetAlgoLined
    \small
    \caption{{Approximating  CL $k$-center via RDS.}}
    \label{alg:k-center-RDS}
    \KwIn{ A family of $l$ disjoint CL sets $\mathcal{Y}$, a positive integer $k$ and a distance bound $\eta=2r^{*}$.}
    \KwOut{A set of centers $C$.}

    Initialization: Set $C\leftarrow \emptyset$, $C'\leftarrow\emptyset$ and $Y'\leftarrow\emptyset$\;
    
    \While{true}{

    Set $f_{RDS} \leftarrow 0$\;
    \tcp{$f_{RDS} =0$ indicates there exists no RDS regarding any $Y'\in\mathcal{Y}$ and  $C$.}
    
         \For{each $Y\in\mathcal{Y}$}{
      
           Construct an auxiliary graph $G(Y,C;E)$ regarding $\eta$ according to Def.~\ref{def:AG}\;
           
         \If { $G(Y,C;E)$ contains  an RDS } {
            Update the center set using a maximum RDS  $(Y',C')$: $C\leftarrow C\cup Y'\setminus C'$\;

            Set  $f_{RDS} \leftarrow 1$\;
        }
        }
        \If {$f_{RDS}$ equals 0}{ Return $C$.}
    }
\end{algorithm}

\subsection{Reverse Dominating Set and the Algorithm}
\label{subsec:Reverse}

To facilitate our description, we introduce the following auxiliary bipartite graph, denoted as $G(Y,C;E)$, which allows us to represent the relationships between a CL set $Y$ and $C$.

\begin{defn}
\label{def:AG}
The auxiliary bipartite graph $G(Y,C;E)$ regarding the threshold $\eta $ is a graph with vertex sets  $Y$ and $C$, where the edge set $E$  is as follows: an edge $e(y,z)$ is included in $E$ iff the metric distance $d(y,z)$ between $y \in Y$ and $z \in C$ is bounded by $\eta$, i.e., $d(y,z) \leq \eta$.
\end{defn}
Then, we use  $N_G(Y')$ to denote the set of neighboring points of $Y'$ in $G$, where $Y' \subseteq Y$,
and consequently define \textit{reverse dominating set} (RDS) as below:

\begin{defn}
\label{def:RDS}
For a center set $C$, a CL set $Y$,  and the auxiliary bipartite graph $G(Y,C;E)$ regarding $\eta$, we say $(Y',C')$, $Y'\subseteq Y$ and  $C'=N_G(Y')\subseteq C$,  is a reverse dominating set (RDS) iff
$|C'|<|Y'|$, where $N_G(Y')$ denotes the set of neighboring points of $Y'$ in $G$.
\end{defn}

In particular, for a point $y\in Y$ that is with distance $d\left(y,\,z\right)>2r^{*}$ to any center $z\in C$ in the auxiliary graph $G(Y,C;E)$, $(\{y\},\emptyset)$ is an RDS.

Note that RDS is a special case of Hall violator that it considers only the case $|C'|<|Y'|$. Anyhow, it inherits the $\mathcal{NP}$-hardness of computing a minimum Hall violator \cite{28_cygan2015parameterized}: 
   \textit{ To find an RDS with minimum cardinality is $\mathcal{NP}$-hard.}

In comparison with the above-mentioned $\mathcal{NP}$-hardness, we discover that a maximum RDS (i.e. an RDS with maximized $|Y'|- |C'|$) can be computed in polynomial time. Consequently,
our algorithm proceeds in iterations, where in each iteration it computes a maximum RDS regarding the current center set $C$ and a CL set  $Y\in\cal{Y}$ (if there exists any), and uses the computed RDS to increase $C$ towards a desired solution.   The formal layout of our algorithm is as in Alg.~\ref{alg:k-center-RDS}.

In Fig.~\ref{fig:ProcessofRDS}, we depict an example of RDS, and then demonstrate the augmentation of a center set according to the RDS  in detail.
    Subfig.~\ref{subfig:ProcessofRDS_S1} demonstrates the auxiliary bipartite graph $G$ (see Def.~\ref{def:AG}), where the two sets of points  (in red and black)  respectively correspond to the current center set $C=\{c_1,\,c_2,\,c_3\}$ and a CL set $Y=\{y_1,\,y_2,\,y_3\}$. Note that according to Def.~\ref{def:AG}, a solid edge exists between a center in $C$  and a point in $Y$  iff their distance is bounded by a given $\eta$.
    In Subfig.~\ref{subfig:ProcessofRDS_S2},  an RDS $(Y', C')$ is depicted in the (blue) dashed line rectangle, where $C' = \{c_3\}$ and $Y'=\{y_{2},y_{3}\}$. Apparently, $(Y', C')$ satisfies the criteria of RDS as in Def.~\ref{def:RDS}: $|C'| = 1 < |Y'| = 2$, where $C' = N_G(Y')$ is the set of neighboring points of $Y'$ in $G$. 
   In Subfig.~\ref{subfig:ProcessofRDS_S3}, we lastly augment $C$ with the RDS $(Y', C')$ by 
   {updating the} 
   center set as $C\setminus C' \cup Y'$,  obtaining the augmented new center set $C=\{c_1,\,c_2,\,y_2,\,y_3\}$.

\subsection{Correctness of Alg.~\ref{alg:k-center-RDS}}

The correctness of Alg.~\ref{alg:k-center-RDS} can be stated as follows:
\begin{thm}
\label{thm:correctnessofalg1}
Alg.~\ref{alg:k-center-RDS} always outputs a center set  $C$ which partitions the dataset $P$ with  the following three conditions satisfied: (1) For each point  $p\in P$, there exist a center $c\in C$ with $d(p,c)\le 2r^*$; (2) all the CL constraints are satisfied; (3) the size of $C$ is bounded by $k$.
\end{thm}
According to the procedure of Alg.~\ref{alg:k-center-RDS}, Condition~(1) apparently holds since there exist no edges between any two points with a distance larger than $ 2r^*$.
For Condition~(2),  we have the following property which is in fact a restatement of Hall's marriage theorem:

\begin{lem}\cite{29_hall1935representatives}
\label{lem:feas-nece}
There exists a perfect matching in $G(Y,C;E)$, iff there exists no RDS in $G(Y,C;E)$.
\end{lem}

Recall that Alg.~\ref{alg:k-center-RDS} terminates once it is unable to find any RDS between the current $C$ and any $Y\in \mathcal{Y}$. Then by the above lemma,  this implies the existence of a perfect matching in each $G(Y,C;E)$ regarding the current $C$ and each  $Y\in \cal Y$. Consequently, this ensures the satisfaction of each CL set in $\mathcal{Y}$.
Lastly, for Condition~(3), we have the following lemma, which immediately indicates $|C|\leq k$:
\begin{lem}
\label{lem:bound_k}
Let $C$ be the set of centers output by  Alg.~\ref{alg:k-center-RDS}. Assume that $\mathcal{V}^{*}$ is an optimum solution that consists of a set of clusters. Then for each pair of centers $c_1, c_2$ of $C$,  $\sigma^*(c_1)\neq \sigma^*(c_2)$ must hold, where $\sigma^*(c_i)$ is the center of $c_i$ according to $\mathcal{V}^{*}$.
\end{lem}
\begin{proof}

We show that the lemma holds in each iteration of the algorithm. That is, to show every point of the growing center set $C$ belongs to a different optimal cluster in each iteration.

For the first iteration,  the lemma holds as analyzed below. In the iteration, the algorithm includes $Y$ in the center set $C=\emptyset$. Due to the  CL constraint of  $Y$, any pair of points of $C$ (currently equals $Y$) must be in different clusters of the optimal solution.

Suppose the lemma is true in the  $l$th iteration of the algorithm for the center set $C_l$, where the points of $C_l$ appear in exactly $|C_l|$ different clusters in the optimal solution. Then we need only to show the lemma also holds after the $(l+1)$th iteration which processes $Y$.

To demonstrate this, we consider two cases based on the presence or absence of an RDS in $G(Y,C_l;E)$:
\begin{enumerate}
\item[(1)] $G(Y,C_l;E)$ contains no RDS:

 Following Lem.~\ref{lem:feas-nece}, in the case, there must exist a perfect matching between $Y$ and $C_l$ in $G(Y,C_l;E)$. So the current center set  $C_l$  remains unchanged in the $(l+1)$th iteration that $C_{l+1}=C_l$,  and hence the lemma obviously holds.

\item[(2)] Otherwise, center set $C_{l+1}$ is updated by Step 7 of Alg.~\ref{alg:k-center-RDS}:

 Assume that $(Y', C')$ is an RDS  in $G(Y,C_l;E)$ with respect to $C_l$ and $Y$. Firstly, by  the induction hypothesis, the points of $C_l\setminus C'$ must appear in $|C_l\setminus C'|$ different clusters of $\mathcal{V}^{*}$. Secondly, due to the CL constraint of  $Y$, the points of $Y'$ must appear in $|Y'|$  different clusters of the optimal clustering $\mathcal{V}^{*}$. Lastly, by the   RDS definition, the distance between every point of $C\setminus C'$ and any point of $Y'$ must exceed $2r^*$. Thus, any center in $C\setminus C'$ can not share an identical cluster (in  $\mathcal{V}^{*}$) with any point of $Y'$, and \textit{vice versa}. Hence, the points within $Y'\cup C\setminus C'$ exactly appear in $|Y'\cup C\setminus C'|$ different clusters of $\mathcal{V}^{*}$.\end{enumerate}
 This completes the proof.
\end{proof}

\begin{lem}\label{lem:Alg1runtime}Let $t_{RDS}$ be the time of finding an RDS or determining there exists no RDS in a given graph.
Then Alg.~\ref{alg:k-center-RDS} deserves a runtime complexity of  $O(kl\cdot t_{RDS})=O(nk\cdot t_{RDS})${, where $l = |\mathcal{Y}|$ denotes the number of sets in $\cal Y$}.  
\end{lem}
\begin{proof} 
  In Alg.~\ref{alg:k-center-RDS}, each for-loop finds an RDS for each $Y\in \cal Y$, and hence consumes $O(lt_{RDS})$. 
Moreover, Alg.~\ref{alg:k-center-RDS} repeats the content of while-loop for $O(k)$ iterations, each of which executes a for-loop. That is because the size of $C$ is bounded by $k$ following Thm.~\ref{thm:correctnessofalg1}, while on the other hand, each for-loop (except the last one) finds at least one RDS and increases the size of $C$ for at least one (otherwise $f_{RDS}=0$ holds and the algorithm terminates). 
Therefore,  the total runtime is $O(kl\cdot t_{RDS})=O(nk\cdot t_{RDS})$.
\end{proof}

\subsection{ Computation of RDS using LP}
\label{subsec:Efficient}
It remains to give a method to compute a maximum RDS. For this task, we propose a linear programming (LP) relaxation and then show that any basic solution of the LP is integral. Let $C$ be the set of current centers in an iteration and $Y$ be a CL set. Then we can relax the task of finding a maximum RDS $(Y',\,C')$ (i.e. an RDS with maximized $|Y'|- |C'|$) as in the following linear program (LP (1)):
{
\begin{align*}
\max &  & \sum_{y\in Y}y-\sum_{z\in C}z\\
s.t. &  & y-z  \leq 0 &&  & \forall e(y,\,z)\in E\\
 &  & 0\leq y,z  \leq1 &&  & \forall y\in Y,\,z\in C
\end{align*}
}
Note that when we force $y,z\in\{0,1\}$, the above LP~(1) is an integer linear program that exactly models the task of finding a maximum RDS. Moreover, we observe that any basic solution of LP~(1) is integral, as stated below:
\begin{lem}
\label{lem:integeral}
 Any feasible basic solution of LP~(1) in integral. In other words,  the values of every $y$ and $z$ in a feasible basic solution of LP~(1) must be integers.
\end{lem}
\begin{proof}
Suppose the lemma is not true.  We can assume without loss of generality that  $0<y<1$ and $0<z<1$ for all $y$ and $z$. Because otherwise we can simply replace $y$ (or $z$) by 0 (or 1) when $y$ (or $z$) equals 0 or 1.
We construct a graph $E_{f}\subseteq E$ (initially an empty set) simply as below:  For each active constraint $y-z=0$ that is linear independent with the constraints already with the corresponding edges in $ E_{f}$, add $e(y,z)$ to $ E_{f}$.
Then the number of active linear independent constraints is clearly bounded by $|E_{f}|$. On the other hand, there exists no cycle
in $E_{f}$, because any cycle in $E_{f}$ indicates a set of linear dependent constraints. Then, $E_{f}$ is a forest, and hence the number of points (variables) in $E_{f}$
is at least $|E_{f}|+1$. According to the definition of a basic solution in linear programming, there must be at least  $|E_{f}|+1$  linearly independent active constraints. This contradicts the fact that the number of such constraints
is bounded by $|E_{f}|$.  This completes the proof.
\end{proof}

From the above lemma, we can obtain a maximum RDS by computing an optimal basic (fractional) solution to LP~(1). It is worth noting that LPs can be efficiently solved in polynomial time using widely-used LP solvers like CPLEX~\cite{30_cplex22_1}. Consequently, we immediately derive the first polynomial-time 2-approximation for CL $k$-center.

{
\section{Acceleration of Computation by Greedy Algorithm with Provable Ratio }

In this section, we first introduce a greedy algorithm designed to accelerate the computation of the \textit{reverse dominating set} (RDS), thereby improving the theoretical runtime of Alg.~\ref{alg:k-center-RDS} mainly incurred by solving the LP formulation. We then demonstrate that the greedy algorithm can successfully find a maximum RDS, should one exist.

\subsection{The greedy algorithm for finding maximum RDS}
\label{subsec:greedy_improve_RDS}
Regarding the current processing $Y$, the key idea of our greedy algorithm is first to find a maximum matching in the auxiliary graph $G(Y,C;E)$, and then starting from the unmatched points of $Y$ to grow an RDS according to the matching, until it arguably becomes a maximum RDS.

The main steps of the algorithm proceed as follows. First, we construct graph $G(Y,C;E)$ and identify a maximum matching $M$ therein. Then, we initialize $Y^\prime\leftarrow Y\setminus M$, which represents the set of uncovered points of $Y$, and  subsequently set  $C^\prime\leftarrow N_G(Y^\prime)$, where recall that $N_G(Y^\prime)$ denotes the set of neighbors of $Y^\prime$ in $G$.   To further expand the center set $C$ to the maximum, we iteratively grow  $Y^\prime$ and $C^\prime$ as below: (1) Update $Y^\prime\leftarrow Y^\prime\cup N_M(C^\prime)$, where $N_M(C^\prime)$ denotes the set of neighbors of $C^\prime$ in maximum matching $M$; (2) Then, set $C^\prime\leftarrow C^\prime\cup N_G(Y^\prime)$. The detailed algorithm is as illustrated in Alg.~\ref{alg:fastgreedy}. 

\begin{algorithm}[t]
\small
\caption{{Greedy algorithm for maximum RDS.}}
\label{alg:fastgreedy}
\KwIn{Center set $C$ and CL set $Y$.}
\KwOut{An RDS $(Y',\,C')$.}

 Construct the auxiliary graph $G(Y,C;E)$ according to Def.~\ref{def:AG} and set $C'\leftarrow\emptyset$\;
 
 Find a maximum matching in $G$, say $M$ 
 and set $Y'\leftarrow Y \setminus M$\;

 \While{$C'$ does not equal $N_G(Y')$}{

 \tcp{$N_G(Y')$ denotes the set of neighbors of $Y'$ in $G$.}
         Set $C'\leftarrow C'\cup N_G(Y')$ and then $Y' \leftarrow Y'\cup N_M(C')$, where $N_M(C')$ is the set of neighbours of $C'$ in $M$;
}

Return $(Y', C')$ as the maximum RDS.

\end{algorithm}

\subsection{Correctness proof and runtime analysis}
\label{subsec:corre_proof}

For convenience, we denote the set of $Y'$ and $C'$ in the $j$th iteration as  $Y'_{(j)}$ and $ C'_{(j)}$, respectively. Moreover,  $y_j$ and $c_j$ represent two points of $Y'_{(j)}\setminus Y'_{(j-1)}$ and $ C'_{(j)}\setminus  C'_{(j-1)}$, respectively. Additionally, we define $\Delta_j=C'_{(j)}\setminus C'_{(j-1)}$.

\begin{lem}\label{lem:CprimeinM}
    During the while-loop procedure, $C'\subseteq M$ always holds. 
\end{lem}

\begin{proof}
    Suppose the lemma is not true. Then assume that $h$ is the iteration during which the first point out of $M$, say $c_h\notin M$, is added to $C'$ as a neighbor of $y_{h-1}$. Moreover, according to the algorithm, $y_{h-1}$ is added to $Y'$ as a neighbor of some $c_{h-1}$ within $M$ in previous iterations. Consequently and repeatedly, we can obtain a sequence of points as below: 
    \begin{equation}
         c_h-y_{h-1}-c_{h-1}-\dots-y_j-c_j-\dots-y_0, \label{path-aug}
    \end{equation}
   where $y_0$ is a point in $Y'_{(0)}=Y\setminus M$. Note  that each pair $(y_j,c_j)$ indicates  an edge appearing in $M$, whereas each pair  $(y_j,c_{j+1})$ constitutes  an edge of $G(Y,C;E)$. Moreover,
   acknowledging  that both $y_0\notin M$ and $c_h\notin M$, it follows that  $M\setminus \{(y_j,c_j)\mid j=1,\dots,h-1\}\cup \{(y_j,c_{j+1})\mid j=0,\dots,h-1\}$ is a matching in $G$, but with a larger size. That is, the path outlined in Equation (\ref{path-aug}) is an augmenting path for $M$ in $G$. This contradicts the assumption that $M$ is a maximum matching in $G$. \end{proof}

\begin{lem}\label{sizeofrds}
    After the while-loop, $(Y', C')$ is an RDS in which $|Y'|-|C'|=|Y|-|M|$. 
\end{lem}

\begin{proof}
    We analyse the size of $|Y'|-|C'|$  in iterations by induction. 
    Before the  while-loop, we have $|Y'_{(0)}|=|Y|-|M|$ and $|C'_{(0)}|=0$, so  $|Y'_{(0)}|-|C'_{(0)}|=|Y|-|M|$ obviously holds. Then, we need only to show the value of $|Y'_{(j)}|-|C'_{(j)}|$ remains unchanged for each $j$ while-loop iteration by induction.

    Assume that the lemma holds after the ($j-1$)th iteration, i.e.  $|Y'_{(j-1)}|-|C'_{(j-1)}| = |Y|-|M|$. Our task is then to show that the lemma remains valid in the $j$th iteration. 
    
    In the $j$th iteration, the algorithm adds $ N_G(Y'_{(j)})$  to $C'_{(j-1)}$ and thus obtains $C'_{(j)}$. Let $\Delta_j=C'_{(j)}\setminus C'_{(j-1)}$. Then we have  
    \begin{equation}
            N_M({\Delta_j})\cap Y'_{(j-1)}=\emptyset.\label{eq:empty}
    \end{equation}
  
    This is because if there were points in $N_M({\Delta_j}) \cap Y'_{(j-1)}$, they should have been added before the $j$th iteration, which contradicts the fact that the points of $\Delta_j$ are added in $j$th iteration. Then following Eq.~(\ref{eq:empty}), we have
\begin{equation}
    |Y'_{(j)}|= |Y'_{(j-1)}\cup N_M({\Delta_j})| = |Y'_{(j-1)}| + |N_M({\Delta_j})|. \label{eq:Y'j}
\end{equation}

    On the other hand, following Lem.~\ref{lem:CprimeinM}, we have  $\Delta_j\subseteq C'_{(j)}\subseteq M$. That is, 
       \begin{equation}
           \vert N_M({\Delta_j}) \vert = |\Delta_j|. \label{eq:number}
       \end{equation}
   Then, following Eq.~\ref{eq:Y'j}, we have 
\begin{align*}
      |Y'_{(j)}| - |C'_{(j)}| 
    &= \left(|Y'_{(j-1)}| + |N_M({\Delta_j})| \right)- \left(|C'_{(j-1)}|+|\Delta_j|\right)\\
    &= \vert Y'_{(j-1)}\vert -\vert C'_{(j-1)}\vert
\end{align*}
where the second equation is from Eq.~\ref{eq:number}.  \end{proof}

\begin{lem}\label{lem:2Algruntime}
The runtime of Alg.~\ref{alg:fastgreedy} is $O(k^{2})$.
\end{lem}
\begin{proof}
    Firstly, we calculate the runtime of Alg.~\ref{alg:fastgreedy}.  Clearly, Step 1 takes $O(|C|+|Y|)$ time to construct the auxiliary graph $G(Y, C; E)$. Then, Step 2 requires $O(|E(G)|+|V(G)|^\frac{3}{2})=O(k^{2})$ time to compute the maximum matching in $G$ using the algorithm by Brand et al.~\cite{31_van2020bipartite}. Moreover, the while-loop repeats for $O(k)$ iterations, with each iteration consuming $O(k)$ time. Therefore, the total runtime of Alg.~\ref{alg:fastgreedy} is $O(k^{2})$.
\end{proof}

Finally, by recalling that Alg.~\ref{alg:k-center-RDS} utilizes Alg.~\ref{alg:fastgreedy} to compute the RDS and combining this with Lem.~\ref{lem:Alg1runtime}, we deduce that the runtime of Alg.~\ref{alg:k-center-RDS} is $O(nk^3)$.

}

\begin{figure}[t]
    \centering

    \includegraphics[width=0.36\textwidth]{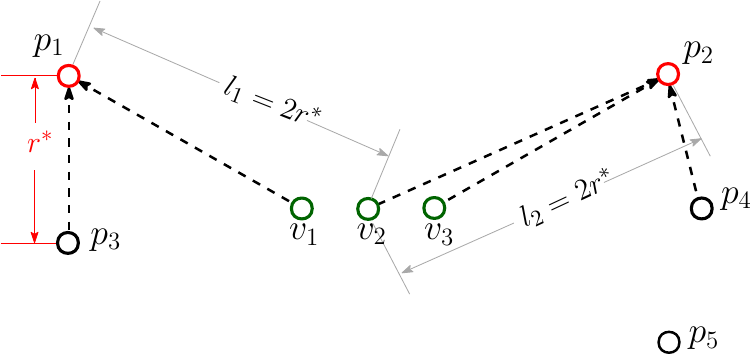}
    \caption{An example of dealing with points of $X$ individually v.s. $X$ as a ``big'' point. In the instance, $C=\{p_1,\,p_2\}$ is the current center set and  $X=\{v_1,\,v_2,\,v_3\}$ is an  ML set.   When considering $X$ as a whole  ``big'' point using the refined distance as in Def.~\ref{def:distanceforX}, we have  $\overline{d}(X,\,C)>\eta=2r^*$, i.e., the point $x$ representing $X$ should be added to $C$ as a new center. In contrast, when using the traditional distance metric,  $x$ (representing $X$) will not be added to C as a new center, since the traditional distances from  $X$  to $p_1$ and $p_2$  are both bounded by $\eta=2r^*$, leading to a radius larger than  $\eta=2r^*$ regardless $X$ is assigned to $p_1$ or $p_2$. }
    \label{fig:MLcounterexamples}
\end{figure}

\section{Whole Algorithm for ML/CL \texorpdfstring{$k$}{k}-Center}
\label{sec:whole}

We will first demonstrate that Alg.~\ref{alg:k-center-RDS} can be extended to approximate $k$-center with both ML and CL constraints, and then show that the algorithm can be easily adjusted by employing binary search when the value of $r^*$ is unknown.

\subsection{ML/CL \texorpdfstring{$k$}{k}-Center with \texorpdfstring{$r^\ast$}{r*}}
For ML/CL \textit{k}-center, the key idea is to contract the ML sets into a set of ``big'' points. Specifically, for each set $X$, we remove all the points from $P$ that belong to $X$ and replace them with a ``big'' point $x$. Then, for distances involving the points resulting from this contraction, we use the following definition:
\begin{defn}
\label{def:distanceforX}
For two points  $x_i$ and $x_j$ that  result from contraction, corresponding to $X_i$ and $X_j$ respectively, the refined distance between them is defined as
\[\hat{d}(x_i,\,x_j)=\hat{d}(X_i,\,X_j)=\max_{p\in X_i,\,q\in X_j}d(p, q).\]
\end{defn}

Treating a single point as a singleton ML set, we can calculate the distance between $x$ (representing $X$) and a point $q\notin X$ using Def.~\ref{def:distanceforX} as follows:
\[\hat{d}(x,\,q)=\hat{d}(X,\,q)=\max_{p\in X}d(p, q).\]

\begin{lem}
    The distance $\hat{d}$ defined in Def.~\ref{def:distanceforX} satisfies the triangle inequality. That is, for any three points $x_1$, $x_2$ and $x_3$ that respectively representing three ML sets $X_1$ $X_2$ and $X_3$, we have:
    \[\hat{d}(x_{i_1},x_{i_2})\leq \hat{d}(x_{i_1},x_{i_3}) +\hat{d}(x_{i_2},x_{i_3}),\]
    where $i_1, i_2, i_3\in \{1,2,3\}$ and $i_1\neq i_2 \neq i_3$.
\end{lem}

\begin{proof}Let $u_i^j= {\arg\max}_{v\in X_{i}} \hat d(X_{j},v)$. That is, $u_i^j$ is the point in $X_i$ that attains the maximum refined distance from $X_{j}$, i.e.,  $\hat d(X_{j},u_i^j)$ attains a maximum.  Conversely, we have the notation $u_j^i$. Then  $u_i^j\in X_i$ and $u_j^i\in X_j$ are two points of $X_i$ and $X_j$ for which $d(u_i^j, u_j^i)$ attains maximum.

According to the nature of Def.~\ref{def:distanceforX}, we have 
\begin{align}
    \hat{d}(x_{i_1},x_{i_3}) = d(u_{i_1}^{i_3},u_{i_3}^{i_1}) \geq d(u_{i_1}^{i_2},u_{i_3}^{i_2})\label{eq:x13}
\end{align} 
and 
\begin{align}
\hat{d}(x_{i_2},x_{i_3}) = d(u_{i_2}^{i_3},u_{i_3}^{i_2}) \geq d(u_{i_2}^{i_1},u_{i_3}^{i_2})\label{eq:x23}
\end{align}

Combining the two inequations above, we have
\begin{align*}
\hat{d}(x_{i_1},x_{i_3}) +\hat{d}(x_{i_2},x_{i_3}) & = d(u_{i_1}^{i_3},u_{i_3}^{i_1})+  d(u_{i_2}^{i_3},u_{i_3}^{i_2}) \\
&\geq  d(u_{i_1}^{i_2},u_{i_3}^{i_2}) + d(u_{i_2}^{i_1},u_{i_3}^{i_2}) \\
&\geq d(u_{i_1}^{i_2}, u_{i_2}^{i_1}) = \hat{d}( x_{i_1},  x_{i_2}),
\end{align*}
where the first inequality is obtained by combining Inequality (\ref{eq:x13}) and (\ref{eq:x23}), the second inequality is from the fact that the original $d$ satisfies the \textit{triangle inequality}. Therefore, the refined distance $\hat{d}$ also satisfies the \textit{triangle inequality}. 
\end{proof}

Then the distance between $x$ and the  center set $C$ is defined as:
\[\overline{d}(x,\,C)=  \min_{q\in C}\hat{d}(X,q) =  \min_{q\in C}\max_{p\in X}d(p,q).\]
By Def.~\ref{def:distanceforX}, we can simply extend Alg.~\ref{alg:k-center-RDS} to solve ML/CL $k$-center. Note that Def.~\ref{def:distanceforX}  is essential for obtaining a desirable solution, as simple heuristic algorithms, without using Def.~\ref{def:distanceforX}, can not lead to desirable solutions even for ML $k$-center alone (as illustrated in Fig.~\ref{fig:MLcounterexamples}).

\begin{algorithm}[t]
\small
\SetAlgoLined
\caption{Whole algorithm for ML/CL $k$-center.}
\label{alg:whole}
\KwIn{Database $P$ of size $n$ with ML sets ${\mathcal{X}}$ and CL sets ${\mathcal{Y}}$ and a positive integer $k$.}
\KwOut{A set of centers $C$.}

Initialization:  Shrink each ML set  $X\in\cal X$ as $x$\;

Compute $\Psi=\{d(p_i,p_j)\mid p_i,\,p_j\in P\}$, the set of distances between each pair points of $P$\;

 \While{true}{
 \tcc{\footnotesize{Each while-loop computes a center set $C$ and  exclude part of the elements from $\Psi$  according  to  $\eta$ and the size of $C$.}}
  Assign the value of the median of $\Psi$ to $\eta$\;
  
  Set $C\leftarrow \emptyset$, $C'\leftarrow\emptyset$ and  $Y'\leftarrow\emptyset$\; 

 Call Alg.~\ref{alg:k-center-RDS} with $\eta$ to compute the center set $C$, except for computing the distances concerning the shrunken point $x$  according to Def.~\ref{def:distanceforX}\;
       
    \eIf{$|C|>k$}{
        Remove each $d \leq \eta$ (except $\eta$) from $\Psi$\;
    }{
        Remove each $d\geq  \eta $ (except $\eta $) from $\Psi$\;
    }
    \tcc{\footnotesize{Note that $|C|\leq k$ indicates $\eta $ is sufficiently large, while $|C|>k$ indicates otherwise.}}
    \If{$|\Psi|= 1$}{
        Return $\eta$ together with the corresponding $C$.
    }
 }
\end{algorithm}

\subsection{Dealing with \texorpdfstring{$r^\ast$}{r*}}

We show the same ratio can be achieved even without knowing $r^*$ based on the following observation:

\begin{lem}
\label{lem:rlays}
Let $\Psi=\{d(p_i,p_j)\vert p_{i},p_{j}\in P\}$ and $r^*$ be optimum radius. Then we have $r^{*}\in \Psi$, i.e., the value $r^*$  must equal to a distance within $\Psi$.
\end{lem}

\begin{proof}
    Suppose the lemma is not true. Clearly, there must exist at least a distance in  $\Psi$  not larger than $r^*$, since otherwise, any pair of points cannot be in the same cluster. Let $r=\max\{d_{ij}\mid d\in\Psi,\,d<r^*\}$. Clearly, under the distances $r$ and $r^*$, every center can exactly cover the same set of points. This is because if a point is covered by center $i$ under distance  $r^*$, it can also be covered by  $i$ under distance $r$ (since $d_{ij}\leq r$ if $d_{ij}<r^*$), and vice versa. Therefore, we can use $r < r^*$ as a smaller clustering radius, leading to a contradiction.
\end{proof}

That is, we need only to find the smallest $r\in \Psi$, such that regarding $2r$, Alg.~\ref{alg:k-center-RDS}  can successfully return $C$ with $\vert C\vert\leq k$.  By employing a binary search on the distances in $\Psi$, we can find in $O(\log n)$ iterations the smallest $r$ under which Alg.~2 can find a feasible solution. Therefore, by combining the two aforementioned techniques, the whole algorithm for $k$-center with both ML and CL sets but without known $r^*$ is depicted in Alg.~\ref{alg:whole}. Eventually, we have  runtime and performance guarantees for  Alg.~\ref{alg:whole} as follows:

\begin{thm}
\label{thm:MLCLratio}
Alg.~\ref{alg:whole} solves the ML/CL $k$-center within runtime $O(nk^{3}\log n)$ and outputs a center set $C$, such that: (1) $d(p,\sigma (p))\leq 2r^*$ holds for $\forall p\in P$ where $\sigma (p)$ is the center for $p$ in $C$; (2) All the ML and CL constraints are satisfied; (3) $|C|\leq k$.
\end{thm}

\begin{proof}

The runtime analysis of Thm.~\ref{thm:MLCLratio} can be easily derived from Lem.~\ref{lem:2Algruntime} by including a multiplicative factor of $O(\log n)$ to account for the binary search.

To prove the latter part of Thm.~\ref{thm:MLCLratio}, we follow a similar line as in the proof of Thm.~\ref{thm:correctnessofalg1} but with more sophisticated details.
First, Condition (1)  holds as analyzed below. On the one hand, following Lem.~\ref{lem:rlays}, $r^{*}\in \Psi$ must hold. On the other hand,  the algorithm employs binary search to find the minimum $r\in\Psi$ and uses $\eta=2r$ as the threshold to call Alg.~\ref{alg:k-center-RDS} (Step 6). That is, $\eta\leq 2r^*$. For  Condition (2), the  CL and ML constraints are obviously satisfied according to the procedure of the algorithm.
Lastly, for Condition (3), we can easily get a similar version of Lem.~\ref{lem:bound_k} in which the claim $|C|\leq k$ remains true. That is because the centers of $C$ that are produced by our algorithm must appear in pairwise different clusters of an optimal solution, provided that the distance involving an ML set $X\in \cal X$ is processed according to the rules as in Def.~\ref{def:distanceforX}. Therefore, we have the correctness of the theorem.
\end{proof}

\begin{figure*}[t]
    \centering
    \captionsetup{justification=centering}
    \subfloat[Disjoint ML/CL]{
        \includegraphics[width=\textwidth]{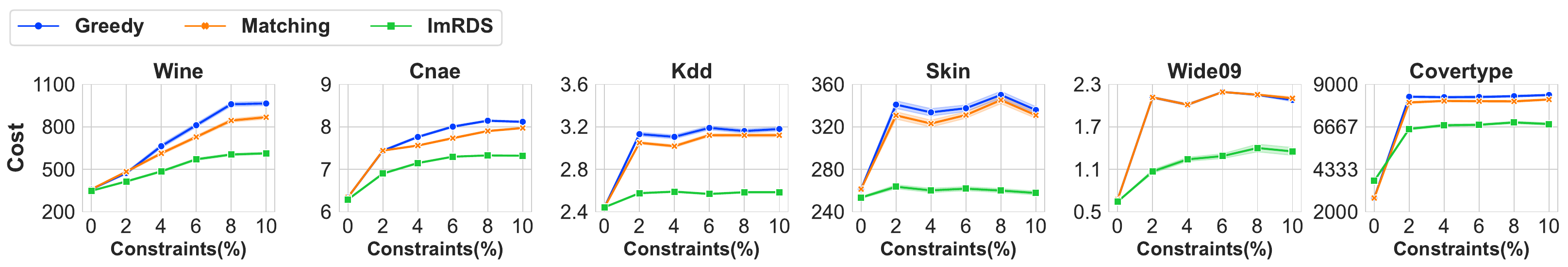}
        \label{subfig:Real-world-Data-Cost-disjoint}
    }
    \hfill
    \subfloat[No control on the size/ratio of intersected ML/CL.]{
        \includegraphics[width=\textwidth]{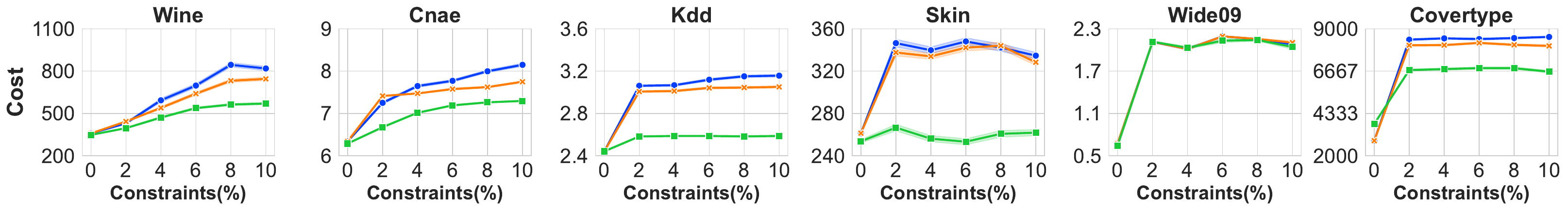}
        \label{subfig:Real-world-Data-Cost-intersected_0}
    }
    \hfill
    \subfloat[5\% of all data points are constrained ML/CL.]{
        \includegraphics[width=\textwidth]{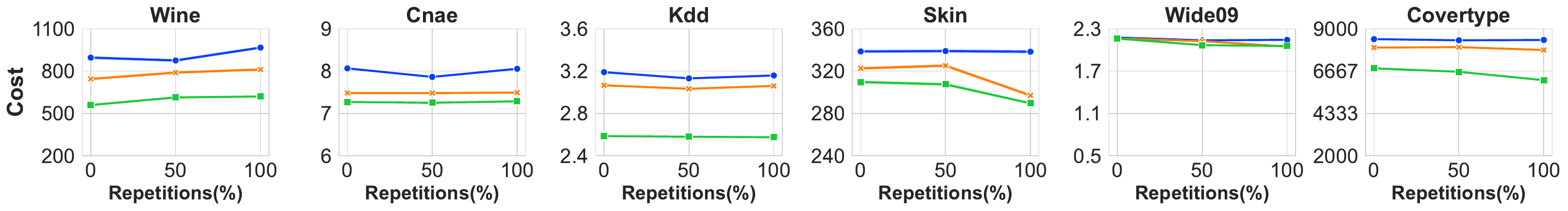}
        \label{subfig:Real-world-Data-Cost-intersected_5}
    }
    \hfill
    \subfloat[10\% of all data points are constrained ML/CL.]{
        \includegraphics[width=\textwidth]{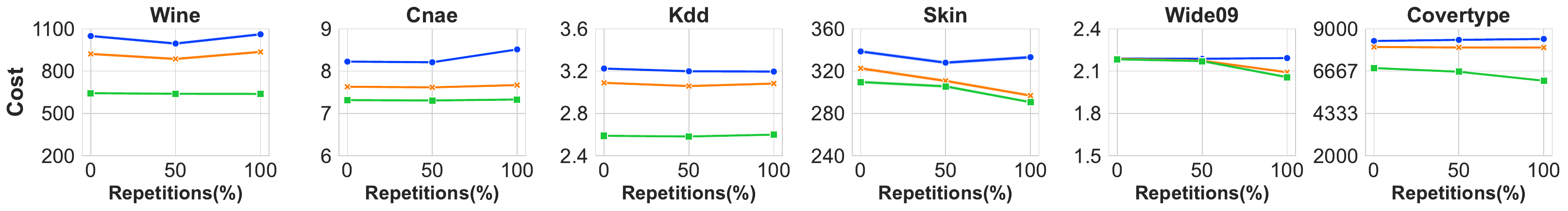}
        \label{subfig:Real-world-Data-Cost-intersected_10}
    }
    \caption{Cost.}
    \label{fig:Real-world-Data-Cost-intersected}
\end{figure*}

\section{Experimental Evaluation}
\label{sec:Exp}

\begin{table*}[t]
\centering
\caption{Datasets statistics.\label{tab:datasets}}
\begin{tabular}{l|c|c|c|l|c}
\toprule[1pt]
\textbf{Datasets} & \textbf{\#Records} & \textbf{\#Dim.} & $\mathbf{k}$ & \textbf{Clusters sizes in ground truth} &  \textbf{SD/Means}\\
\midrule[1pt]
{Wine}  & 178 & 13 & 3 & [59,71,48] & 11.50/59.33\\ 
{Cnae-9 } & 1,080 & 856 & 9 &[120,120,120,120,120,120,120,120,120] & 0.0/120\\ 
{NLS-KDD} \ & 22,544 & 41 & 2& [12833,9711] & 2207.59/11272\\
{Skin}  & 245,057 & 3 & 2 & [50859,194198] & 101355.98/122528.5 \\ 
{Wide09 } & 570,223 & 21 & 13 & [340390,158902,70141,676,97,5,4,3,1,1,1,1,1] & 100425.58/43863.31\\
{Covertype} & 581,012 & 54 & 7 & [211840,283301,35754,2747,9493,17367,20510] & 114752.58/3001.71\\
\midrule[1pt]
\end{tabular}
\end{table*}

\subsection{Experimental Configurations}
\label{sec:Exp_config}
This subsection includes a brief description of the experimental configurations as stated below.

\descr{Real-world datasets.} 
We follow existing studies on constrained clustering~\cite{18_wagstaff2001constrained,33_malkomes2015fast} to use the four UCI datasets \footnote{\url{http://archive.ics.uci.edu/ml}} (Wine, Cnae-9, Skin and Covertype) and two famous network traffic datasets (KDDTest+ of NLS-KDD~\footnote{\url{https://kdd.ics.uci.edu/databases/kddcup99/kddcup99.html}} and Wide09 of MAWI~\footnote{\url{https://mawi.wide.ad.jp/mawi/}}) to evaluate the algorithms on the Internet traffic classification problem.

\begin{figure*}[ht]
    \centering
    \captionsetup{justification=centering}
    \subfloat[Disjoint ML/CL]{
        \includegraphics[width=\textwidth]{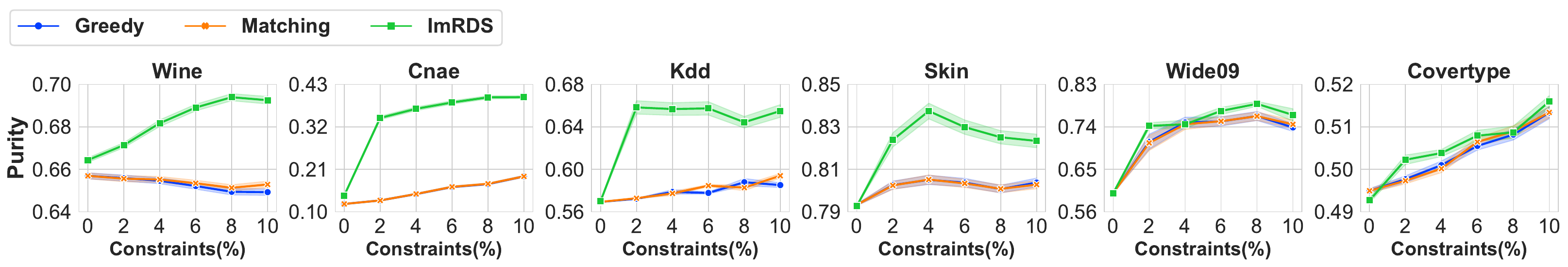}
        \label{subfig:Real-world-Data-Purity-disjoint}
    }
    \hfill
    \subfloat[No control on the size/ratio of intersected ML/CL.]{
        \includegraphics[width=\textwidth]{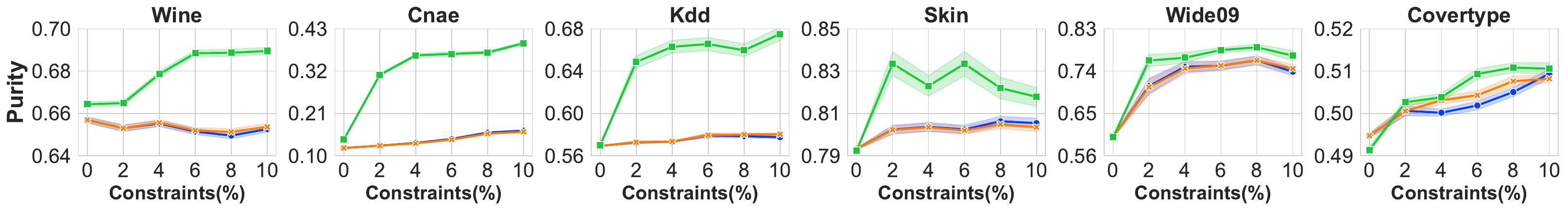}
        \label{subfig:Real-world-Data-Purity-intersected_0}
    }
    \hfill
    \subfloat[5\% of all data points are constrained ML/CL.]{
        \includegraphics[width=\textwidth]{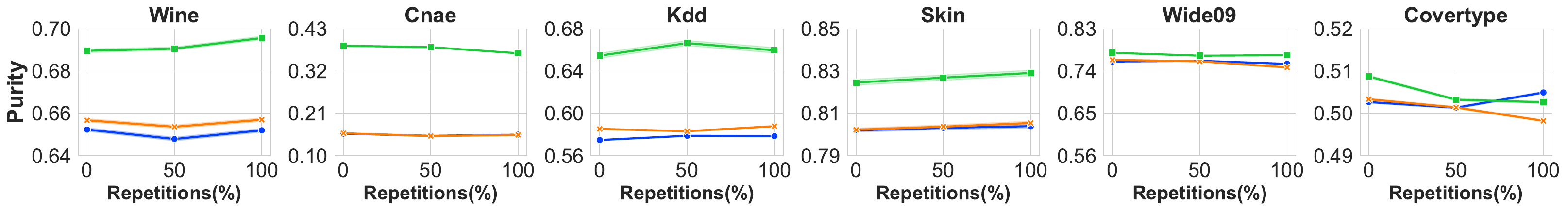}
        \label{subfig:Real-world-Data-Purity-intersected_5}
    }
    \hfill
    \subfloat[10\% of all data points are constrained ML/CL.]{
    \includegraphics[width=\textwidth]{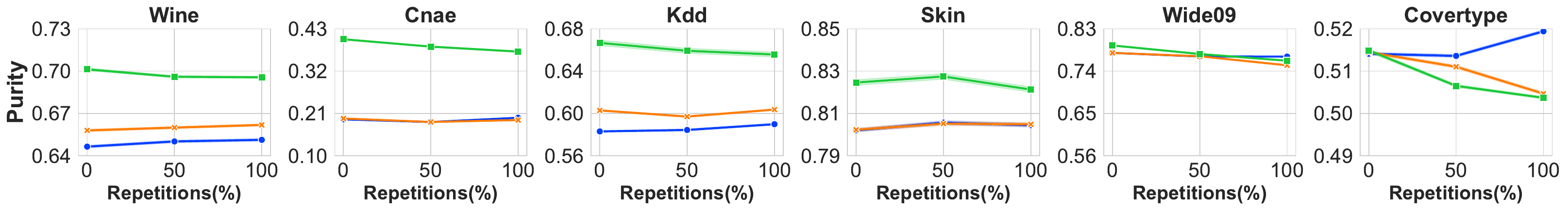}
        \label{subfig:Real-world-Data-Purity-intersected_10}
    }
    \caption{Purity.}
    \label{fig:Real-world-Data-Purity-intersected}
\end{figure*}

\descr{Simulated datasets.} It is challenging for us to evaluate the practical approximation ratio of Alg.~\ref{alg:whole} on real-world datasets due to the lack of optimal solution costs as a benchmark for the constrained $k$-center problem. To overcome this problem, we construct simulated datasets using given parameters ($k$, $n$, and $r^*$), where $k$ is the number of clusters, $n$ is the number of data samples, and $r^{*}$ is the pre-defined optimal radius for the constrained $k$-center problem. Tab.~\ref{tab:datasets} gives brief statistics of the aforementioned datasets.

\descr{Constraints construction.} 
We construct both disjoint and intersected ML/CL constraints for the real-world and the simulated datasets in accordance with Sections~\ref{sec:Intro} and \ref{sec:Alg}, to evaluate the clustering performance of our approximation algorithm against baselines. In short, for a given dataset, a given number of constrained data points, and a given number of participants (who have their own background knowledge of ML/CL), we both uniformly and biased sample data points from the raw dataset into different ML and CL sets. 

\descr{Algorithms.}
Since this is the \emph{first non-trivial} work with a $2$-opt algorithm for the constrained $k$-center problem with disjoint CL sets, we propose two baseline algorithms - a greedy algorithm (Greedy) and a matching-based algorithm (Matching). In brief, Greedy is adapted from a constrained $k$-means algorithm~\cite{18_wagstaff2001constrained} to handle the CL sets while considering the ML constraints as ``big" points, and matching is a simple improvement of Greedy by overall matching points to closer centers that incurs a smaller covering radius.  

All our algorithms operate within the framework of Alg.~\ref{alg:whole} (denoted by Alg.~\ref{alg:whole}), each requiring the computation of a maximum RDS. To compute such a maximum RDS, we introduce two primary algorithms: an LP-based method discussed in Subsec.~\ref{subsec:Efficient} (denoted by LP) and another method based on Alg.~\ref{alg:fastgreedy} described in Subsec.~\ref{subsec:greedy_improve_RDS} (denoted by BaRDS). {Additionally, to further enhance runtime on the large $k$, we developed an approach by refining the computation of RDS (denoted by ImRDS).} Since all three algorithms compute a maximum RDS using the same input and exhibit equivalent performance in terms of clustering accuracy, we present the clustering accuracy of only one (denoted as Alg.~\ref{alg:whole}) for comparison against the baselines.

\descr{Evaluation metrics.} Following existing studies on clustering~\cite{11_bera2022fair,34_wang2014unsupervised,35_rand1971objective}, we use the common clustering quality metrics in the experiments, which are \textit{Cost}~\cite{11_bera2022fair}, \textit{Purity}~\cite{34_wang2014unsupervised}, \textit{Normalized Mutual Information} (NMI), and \textit{Rand Index} (RI)~\cite{35_rand1971objective}. For \textit{runtime}, we report it in the base ten logarithms.

Let $S=\{S_{1}, \dots, S_{k}\}$ and $L=\{L_{1}, \dots, L_{k}\}$ be the set of sets of data samples clustered by their ground truth labels and clustering algorithms, respectively.

\textbf{Cost}~\cite{11_bera2022fair} is the most commonly used metric for evaluating the clustering performance. In this paper, we define the clustering cost as the maximum distance (radius) between any data sample and its assigned center.
\begin{equation}
    \text{cost} = r = \max_{p\in P} d_{c}(p,C),
\end{equation}
where the algorithm selects $C$, and $d_{c}(\cdot,\cdot)$ denotes the distance that is satisfied the constrained relationship from $p$ to its assigned center $c_p \in C$.

\textbf{Purity}~\cite{34_wang2014unsupervised} measures the extent to which clusters contain a single cluster, and it is defined as: 
\begin{equation}
    \text{Purity} = \frac{1}{n} \sum^{k}_{i} \max_{j} |S_{i}\cap L_{j}|.
\end{equation}

\textbf{Normalized Mutual Information} (NMI)~\cite{34_wang2014unsupervised} is a normalized variant of mutual information (MI), which measures the mutual dependence between the clustering results and the ground truth, considering both homogeneity and completeness. The NMI measure is defined as:
\begin{equation}
    \text{NMI} = \frac{\text{MI}(S,L)}{(\text{H}(S)+\text{H}(L))/2},
\end{equation}
where $\text{MI}(X,Y) = \text{H}(X)-\text{H}(X|Y) $ is the mutual information score, $\text{H}(X)=-\sum^{k}_{i=0} P(X_i)\log P(X_i)$ is the Shannon entropy of $X$, and $\text{H}(X|Y)=-\sum^{k}_{i=0} P(X_i)\sum^{k}_{j=0}P(Y_j|X_i) $ is the conditional entropy of $X$ given $Y$.

\textbf{Rand Index} (RI)~\cite{18_wagstaff2001constrained} is a measure of the similarity between pair-wise point clusterings, which is defined by \cite{35_rand1971objective} as:
\begin{equation}
    \text{RI}=\frac{a+b}{n*(n-1)/2},
\end{equation}
where $a$ is the number of pairs of points that have the same label in $S$ and $L$, $b$ is the number of pairs of points that have different labels in $S$ and $L$, and $n$ is the number of all points.

\descr{Implementation details.}
We implemented all algorithms in Java 1.8.0 on a Mac computer equipped with an Apple M1 Max CPU and 32 GB of RAM. We

\begin{figure*}[!ht]
    \centering
    \captionsetup{justification=centering}
    \subfloat[Disjoint ML/CL]{
        \includegraphics[width=\textwidth]{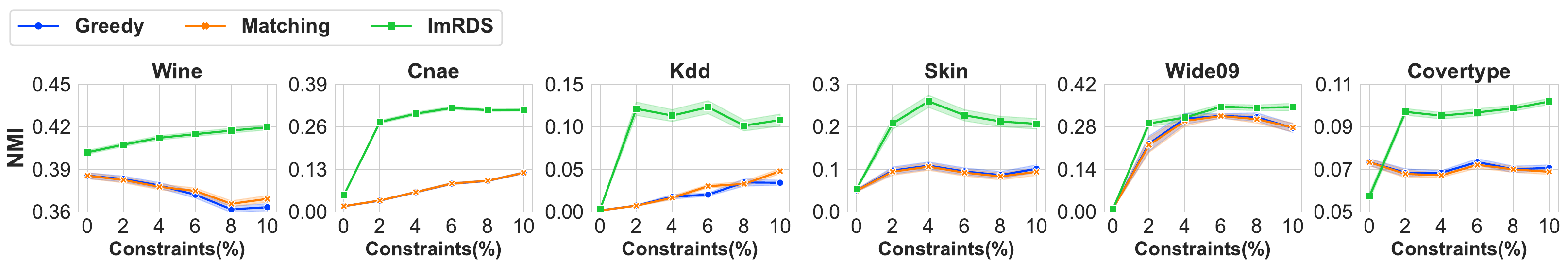}
        \label{subfig:Real-world-Data-NMI-disjoint}
    }
    \hfill
    \subfloat[No control on the size/ratio of intersected ML/CL.]{
        \includegraphics[width=\textwidth]{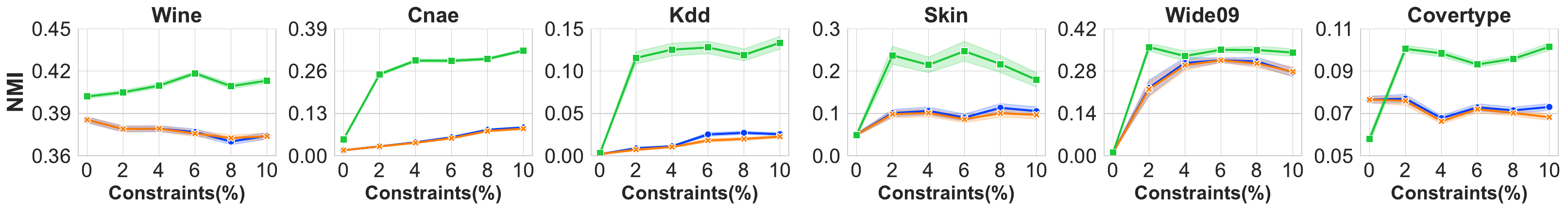}
        \label{subfig:Real-world-Data-NMI-intersected_0}
    }
    \hfill
    \subfloat[5\% of all data points are constrained ML/CL.]{
        \includegraphics[width=\textwidth]{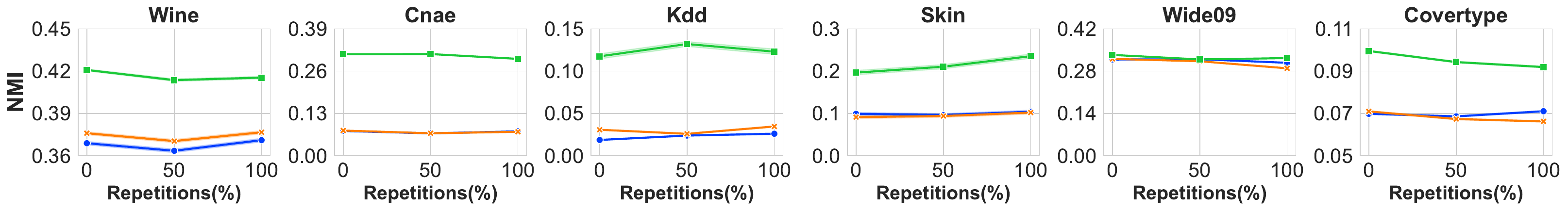}
        \label{subfig:Real-world-Data-NMI-intersected_5}
    }
    \hfill
    \subfloat[10\% of all data points are constrained ML/CL.]{
        \includegraphics[width=\textwidth]{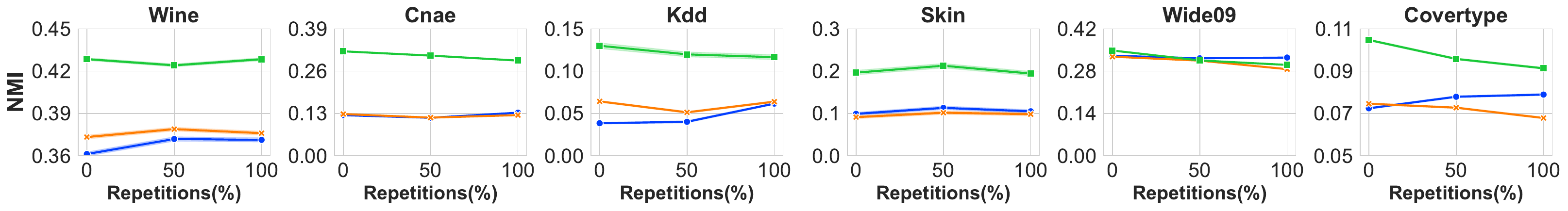}
        \label{subfig:Real-world-Data-NMI-intersected_10}
    }
    \caption{Normalized Mutual Information.}
    \label{fig:Real-world-Data-NMI-intersected}
\end{figure*}

\begin{figure*}[ht]
    \centering
    \captionsetup{justification=centering}
    \subfloat[Disjoint ML/CL.]{
        \includegraphics[width=\textwidth]{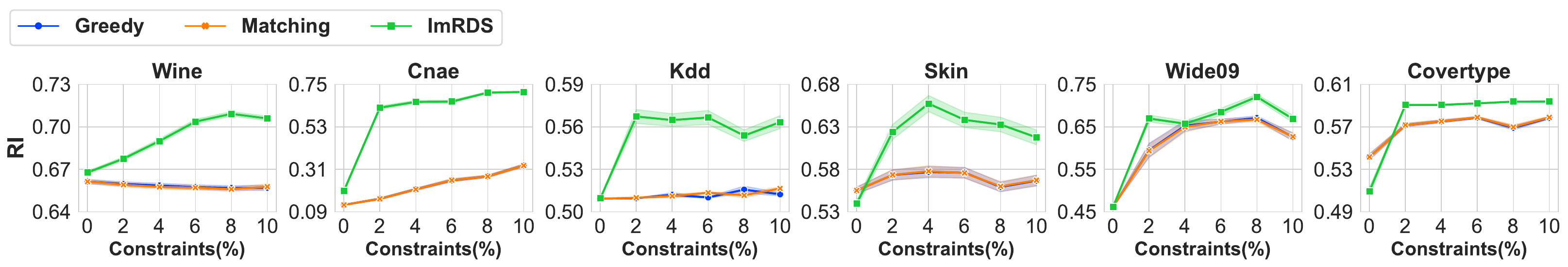}
        \label{subfig:Real-world-Data-RI-disjoint}
    }
    \hfill
    \subfloat[No control on the size/ratio of intersected ML/CL.]{
        \includegraphics[width=\textwidth]{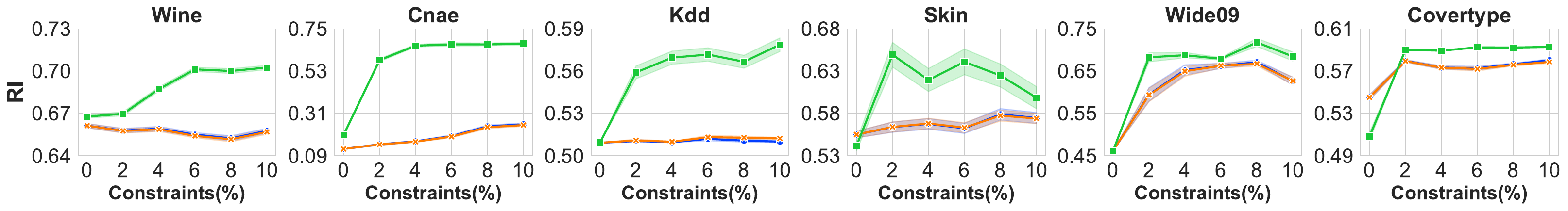}
        \label{subfig:Real-world-Data-RI-intersected_0}
    }
    \hfill
    \subfloat[5\% of all data points are constrained ML/CL.]{
        \includegraphics[width=\textwidth]{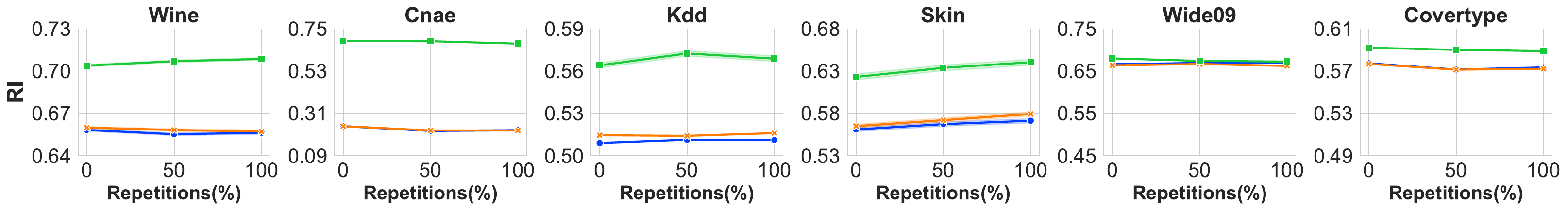}
        \label{subfig:Real-world-Data-RI-intersected_5}
    }
    \hfill
    \subfloat[10\% of all data points are constrained ML/CL.]{
        \includegraphics[width=\textwidth]{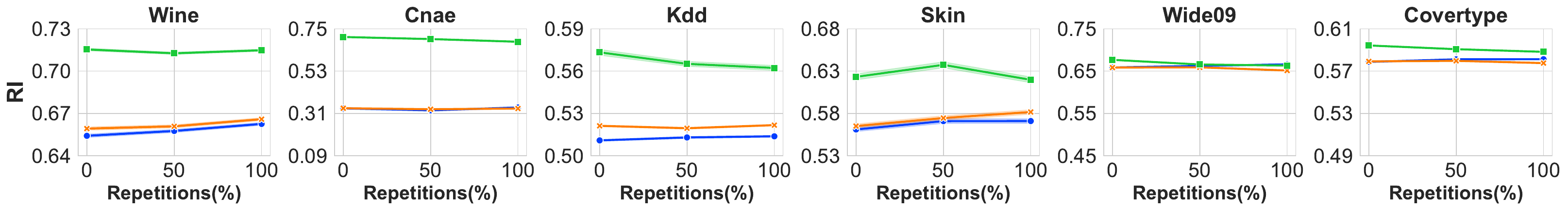}
        \label{subfig:Real-world-Data-RI-intersected_10}
    }
    \caption{Rand Index.}
    \label{fig:Real-world-Data-RI-intersected}
\end{figure*}

\begin{figure*}[ht]
    \centering
    \captionsetup{justification=centering}
    \subfloat[Disjoint ML/CL.]{
        \includegraphics[width=\textwidth]{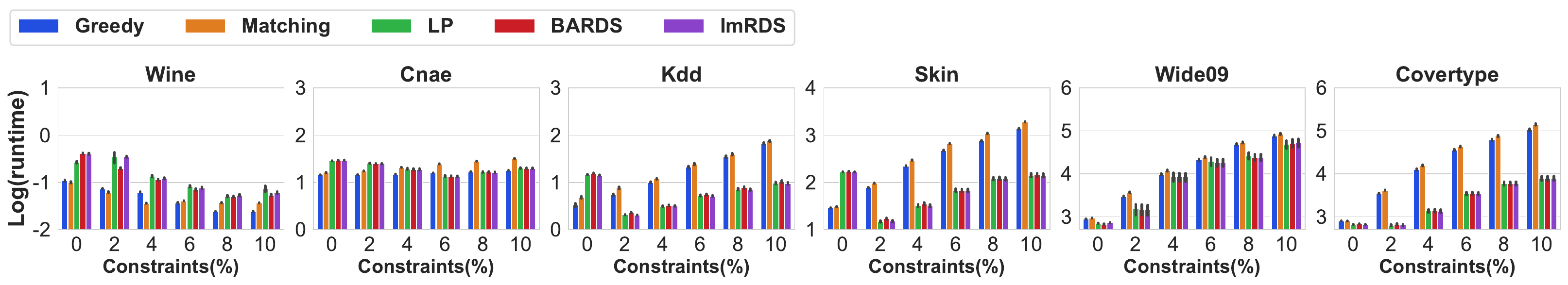}
        \label{subfig:Real-world-Data-runtime-disjoint}
    }
    \hfill
    \subfloat[No control on the size/ratio of intersected ML/CL.]{
        \includegraphics[width=\textwidth]{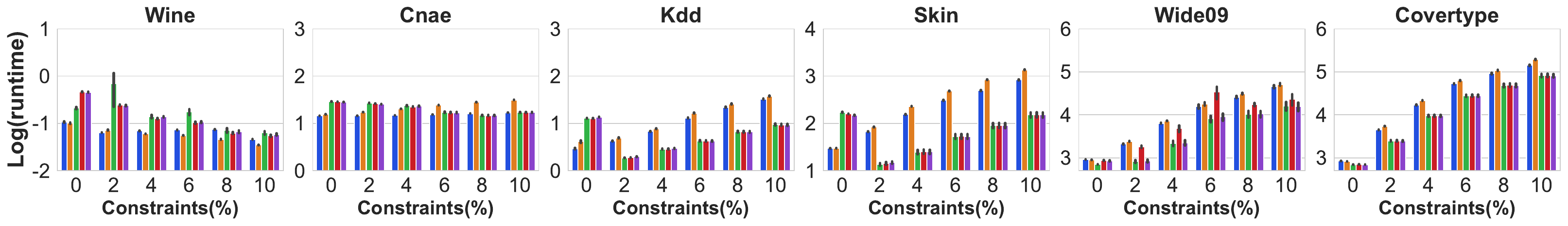}
        \label{subfig:Real-world-Data-runtime-intersected_0}
    }
    \hfill
    \subfloat[5\% of all data points are constrained ML/CL.]{
        \includegraphics[width=\textwidth]{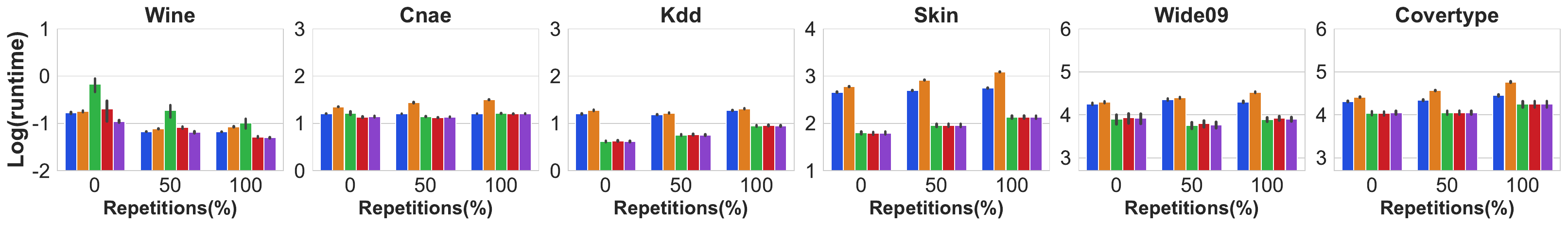}
        \label{subfig:Real-world-Data-runtime-intersected_5}
    }
    \hfill
    \subfloat[10\% of all data points are constrained ML/CL.]{
        \includegraphics[width=\textwidth]{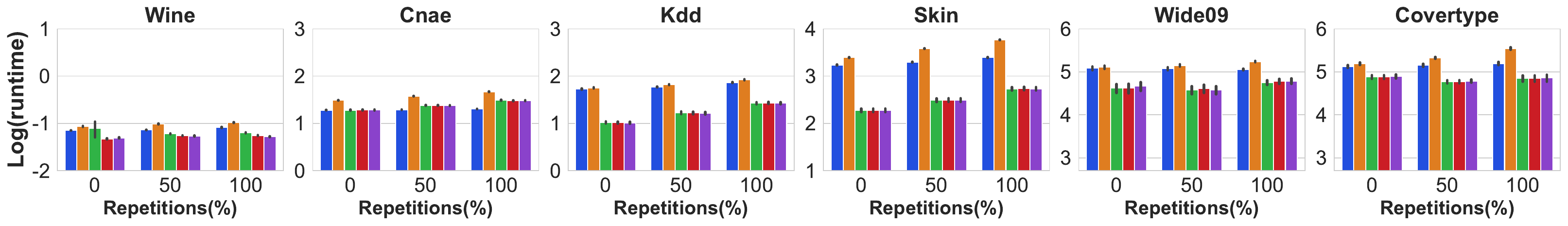}
        \label{subfig:Real-world-Data-runtime-intersected_10}
    }
     \caption{Runtime on the real-world dataset.}
    \label{fig:Real-world-Data-runtime}
\end{figure*}

\subsection{Clustering quality and efficiency with ML/CL}
\label{subsec:Clustering_quality}

In this subsection, we show and analyze the experimental clustering quality and efficiency (averaged over $40$ runs on each dataset) of Alg.~\ref{alg:whole} and the two baselines, when applied to disjoint and intersected ML/CL settings. First, we will detail the settings for each subfigure. Then, by comparing these subfigures, we will report the experimental results.

\descr{Experimental setting.} 
From Fig.~\ref{fig:Real-world-Data-Cost-intersected} to Fig.~\ref{fig:Real-world-Data-runtime}, we present the clustering quality results of  Alg.~\ref{alg:whole} for both disjoint and intersected ML/CL case. Specifically, we demonstrate Alg.~\ref{alg:whole}'s performance on the disjoint ML/CL case in subfigs.~(a), the arbitrary intersected ML/CL case in subfigs.~(b)), settings where 5\%  and 10\% of all data points are constrained under ML/CL constraints in subfigs.~(c) and (d), respectively.

In addition, when working on the intersected ML/CL constraints, it is worth to acknowledge the following issues:
\begin{itemize}[leftmargin=10pt]
    \item 
    First, when CL sets overlap, conflicts may arise among data points in the CL relationship, which are known to be computationally challenging to handle~\cite{19_davidson2007complexity}. In our experiments, we intentionally disregard these conflicts, which means that mis-clustered data points may occur as a result of the conflicts being ignored.

    \item
    Second, in contrast to disjoint constraints, the intersected constrained dataset permits the inclusion of the same points in different CL sets. When the intersected constrained set is provided as input, we treat points with ML constraints as identical (without altering the center-finding step in the algorithm). During the assignment step of the algorithm, we adopt a greedy approach by utilizing the updated results for assignment, regardless of whether a point has already been assigned to a center.

    \item 
    Third, in our experiment, our objective was to thoroughly investigate the impact of intersection on the performance of various clustering algorithms - Greedy, Matching and Alg.~\ref{alg:whole}. To achieve this, we vary the ratio of intersected ML/CL data points relative to a fixed ratio between the constrained data point size and the whole dataset size.
\end{itemize}

\descr{Experimental Analysis.} 
In this subsection, compared with the baselines and different constrained settings, we report the experimental
results and efficiency (averaged over $40$ runs on each dataset) for Alg.~\ref{alg:whole} on six real-world datasets. The summary of our observations is provided below.

\descr{Given disjointed and intersected ML/CL constraints, Alg.~\ref{alg:whole} outperforms the baseline algorithms as guaranteed.} During 0\% to 10\% of all data points as constrained points in Subfig.~\ref{subfig:Real-world-Data-Cost-disjoint}, the cost value of Alg.~\ref{alg:whole} always less than the baselines. From Subfig.~\ref{subfig:Real-world-Data-Purity-disjoint} to Fig.~\ref{subfig:Real-world-Data-RI-disjoint}, while varying the number of constrained data points from $0\%$ to $10\%$ of all data points, we observe that Alg.~\ref{alg:whole} demonstrates promising clustering accuracy across all performance metrics, varying $10\%$ to $300\%$ clustering quality improvement. Since the experimental clustering accuracy is highly related to the theoretical approximation ratio, the guaranteed $2$-opt for the constrained $k$-center problem with disjoint ML/CL constraints provides Alg.~\ref{alg:whole} with a significant advantage.

On intersected CL sets, despite the lack of theoretical guarantees, our algorithm (Alg.~\ref{alg:whole}) demonstrates remarkable clustering performance for the constrained $k$-center problem. As shown in Fig.~\ref{fig:Real-world-Data-Cost-intersected} to Fig.~\ref{fig:Real-world-Data-RI-intersected}, our algorithm (Alg.~\ref{alg:whole}) outperforms the two baselines in all settings - percentage of constrained ML/CL, metrics and datasets.

\begin{enumerate}[leftmargin=10pt]
    \item 
    When varying the percentage of constrained data points from $0\%$ to $10\%$ of the total data points without controlling the size/ratio of intersected ML/CL (results shown in Subfigs.~\ref{subfig:Real-world-Data-Cost-intersected_0}, \ref{subfig:Real-world-Data-Purity-intersected_0}, \ref{subfig:Real-world-Data-NMI-intersected_0} and \ref{subfig:Real-world-Data-RI-intersected_0}), that is through sampling with replacement during constraint construction, our algorithm consistently surpasses the performance of the two baseline algorithms across all four clustering quality metrics (note that constraints exceeding $10\%$ lead to excessive set conflicts). Our analysis suggests that the utilization of constraints in our algorithm enables continuous refinement of center selection and clustering, thereby enhancing its effectiveness even in natural scenarios involving intersected constraint sets.

    \item 
    When fixing the ratio of the constrained ML/CL to $5\%$ and $10\%$ while varying the ratio of intersected constraints from $0\%$ to $50\%$ and to $100\%$ ($5\%$ constraints in Subfigs.~\ref{subfig:Real-world-Data-Cost-intersected_5}, ~\ref{subfig:Real-world-Data-Purity-intersected_5}, \ref{subfig:Real-world-Data-NMI-intersected_5} and \ref{subfig:Real-world-Data-RI-intersected_5}; $10\%$ constraints in Subfigs.~\ref{subfig:Real-world-Data-Cost-intersected_10}, ~\ref{subfig:Real-world-Data-Purity-intersected_10}, \ref{subfig:Real-world-Data-NMI-intersected_10} and \ref{subfig:Real-world-Data-RI-intersected_10}), our algorithm outperforms the two baselines on four clustering quality measurements in almost all cases. However, we observed that the cases where it underperforms occur when 10\% of all data points have ML/CL constraints, with a 100\% ratio of intersected constraints in the Wide09 dataset. We suggest that in high-dimensional datasets, a significant number of intersected CL/ML constraints can invalidate more constraints and cannot improve the effectiveness of the chosen centers for our algorithm.
\end{enumerate}

\descr{Alg.~\ref{alg:whole} demonstrates significantly better performance on sparse datasets.} An example of such a dataset is Cnae-9, which is highly sparse with $99.22\%$  of its entries being zeros. Among all the datasets examined, we observed the largest discrepancy in terms of cost/accuracy between Alg.~\ref{alg:whole} and the two baseline algorithms when applied to the Cnae-9 dataset. For example, after considering the constraints, Alg.~\ref{alg:whole} increase about more than 60\% in Subfig.~\ref{subfig:Real-world-Data-RI-disjoint} and Subfig.~\ref{subfig:Real-world-Data-RI-intersected_0} for RI metric of Cnae-9. We attribute this disparity to the fact that the traditional $k$-center problem struggles with sparse high-dimensional datasets, as highlighted in the study~\cite{36_steinbach2004challenges}. However, the introduction of constraints proves beneficial in adjusting misclustering and center selection, offering improved performance in these cases.

\descr{As the number of constraints increases,  all three algorithms show a general upward trend in clustering accuracy. However, Alg.~\ref{alg:whole} stabilizes once the number of constraints reaches $4\%$.} 
We argue that the reason behind this behavior is that adding constraints effectively reduces the feasible solution space in optimization problems. Introducing more constraints is expected to improve the solution quality of an approximation algorithm. Consequently, we anticipate a higher clustering accuracy for a constrained $k$-center algorithm as the number of constraints increases.

The crucial factor in bounding the approximation ratio for the constrained $k$-center problem with disjoint ML/CL constraints is the selection of centers. Analyzing the results depicted in the figures, we observe that Greedy and Matching algorithms often yield similar clustering outcomes despite employing distinct strategies to handle the CL constraints. We contend that the cluster assignment methods have minimum influence on certain datasets, as the improvement in the experiment hinges on the correction of center selection facilitated by the approximation algorithm with constraints. This correction not only enhances the approximation ratio but also contributes to improved clustering accuracy. 

\begin{figure}[t]
    \centering
    \subfloat[Disjoint ML/CL.]{
        \includegraphics[width=\linewidth]{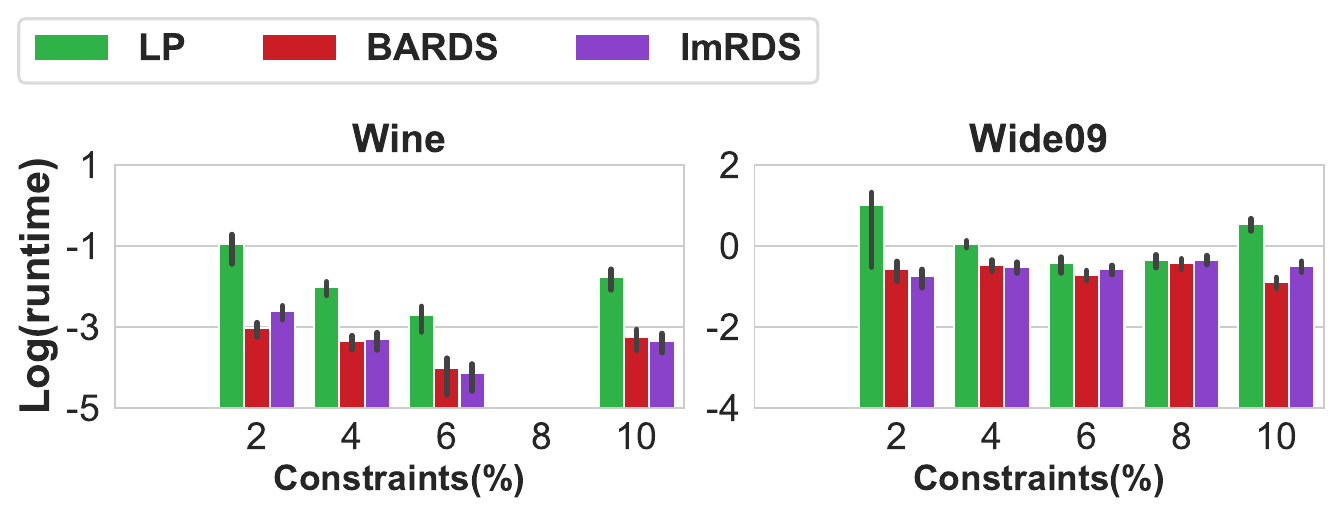}
        \label{subfig:RDS-runtime-disjoint}
    }
    
    \subfloat[No control on the size/ratio of intersected ML/CL.]{
        \includegraphics[width=\linewidth]{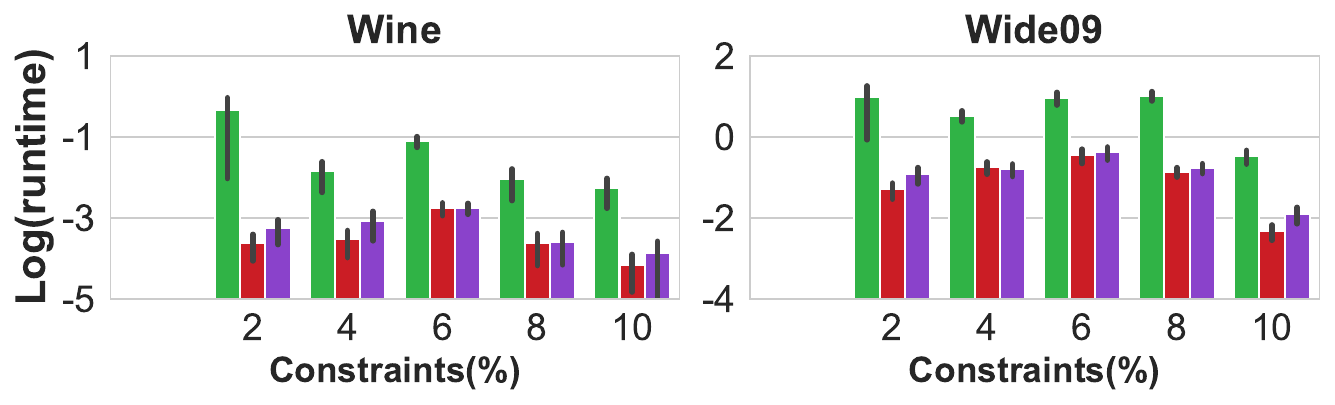}
        \label{subfig:RDS-runtime-intersected_0}
    }
     \caption{Runtime of RDS.}
    \label{fig:RDS-runtime}
\end{figure}

\subsection{Runtime Analysis}
In this subsection, we further analyze the runtime performance depicted in this work. 
First, we compare our algorithms (i.e., LP, BaRDS, ImRDS) with baseline methods under identical settings on real-world datasets, as shown in Fig.~\ref{fig:Real-world-Data-runtime}. Then, we highlight the runtime improvements achieved by ImRDS and BaRDS on both real-world and simulated datasets with larger values of $k$.

Alg.~\ref{alg:whole}  is efficient for solving the constrained $k$-center problem with \emph{disjoint and intersected} ML/CL constraints. As illustrated in Fig.~\ref{fig:Real-world-Data-runtime}, it demonstrates significantly better efficiency than the baseline methods (Matching and Greedy) when running on larger real-world datasets. As the number of constrained points increases, the cases where the maximum size of the CL set contains more points become more likely. Thus, Alg.~\ref{alg:whole} can quickly identify the set of centers, whereas the Greedy and Matching algorithms must rely on traditional methods to select the centers, resulting in slower performance.

In Subfig.~\ref{subfig:Real-world-Data-runtime-disjoint} and Subfig.~\ref{subfig:Real-world-Data-runtime-intersected_0}, Alg.~\ref{alg:whole} exhibits shorter runtimes at various constraint rates compared to the scenario with a 0\% constraint rate for the Wine, Cnae, Kdd, and Skin datasets. We give a similar reason as above. This phenomenon is primarily attributable to the relationship between the value of $k$ and the size of the largest CL set. For datasets with small values of $k$, the size of the largest CL set is often equal to $k$. Consequently, our algorithm initially selects one of the largest CL sets and augments it to form the desired center set directly. In these cases, the first selected CL set immediately forms the center set, removing the necessity for further center selection. Conversely, without constraints (i.e., at 0\% constraint rates), the step of center selection is required, leading to longer runtimes.

As the number of constraints increases, the runtime generally increases for larger datasets in Fig.~\ref{fig:Real-world-Data-runtime}, since all algorithms require more processing time to handle points under CL/ML constraints. However, we observe that the smaller Wine dataset spends more runtime on center selection rather than assignment steps, resulting in significantly longer runtimes for the LP algorithm compared to other methods. Additionally, BaRDS and ImRDS demonstrate runtime improvements over LP, consistent with ImRDS design predictions.

\begin{figure}[t]
    \centering
    \includegraphics[width=\linewidth]{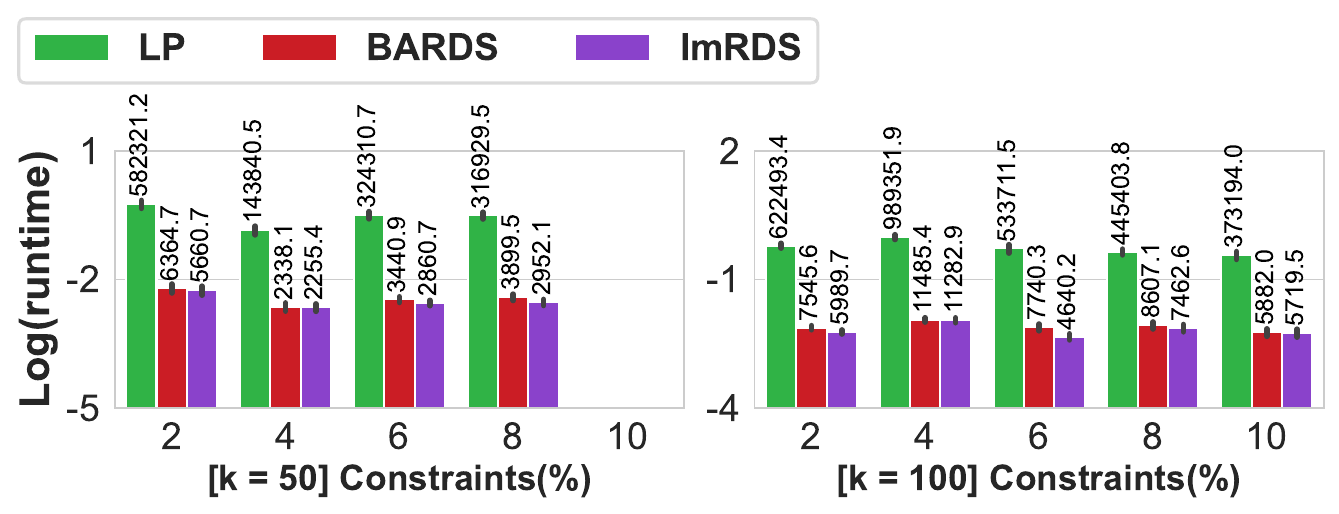}
    \caption{Runtime on the simulated dataset with large $k$ ($k$ = 50 or 100). }
    \label{fig:simu_runtime}
\end{figure}

\descr{Analysis of RDS Runtimes}

According to the above results, we find the algorithms (i.e., LP, BaRDS, ImRDS) exhibit similar running times for most large datasets. 
Therefore, we focus on analyzing the runtime of RDS specifically, by avoiding the influence of the assignment steps. Fig.~\ref{fig:RDS-runtime} presents the experimental results for RDS runtime, comparing the performance of the three algorithms across two real-world datasets. We select these particular datasets because KDD and Skin involve binary classification tasks, while Cnae and Covertype facilitate adjusting the maximum CL set size to equal the parameter $k$. Additionally, the presence of empty entries in this figure indicates that the corresponding algorithms did not utilize RDS for those specific constraint ratios in both Figs.~\ref{fig:RDS-runtime} and~\ref{fig:simu_runtime}.

\textit{In Fig.~\ref{fig:RDS-runtime}, the RDS algorithms utilizing matching (i.e., BaRDS and ImRDS) consistently run faster than the LP method.} This is because the LP method requires an optimizer, resulting in a longer runtime. Furthermore, the two matching-based RDS algorithms (BaRDS and ImRDS) demonstrate similar runtimes, primarily because their performance heavily depends on the parameter $k$, which tends to be similar in real-world datasets. Consequently, distinguishing significant runtime differences between these methods becomes challenging when $k$ is not sufficiently large.

To further evaluate runtime performance, we conducted additional experiments on two simulated datasets with larger values of $k$ (specifically, $k = 50$ and $k = 100$). By varying constraint ratios from 2\% to 10\%, we compared the runtimes of ImRDS, BaRDS, and the LP algorithm, as illustrated in Fig.~\ref{fig:simu_runtime}. Note that the values in the figure represent average runtimes measured in microseconds ($\mu s$).

The goal was to evaluate the RDS runtime of these algorithms under identical settings. The results clearly show that ImRDS consistently outperforms both the LP algorithm and BaRDS in terms of efficiency when clustering datasets with larger values of $k$ across different numbers of disjoint constraints. 

\subsection{Empirical Approximation Ratio}
\label{subsec:Empirical}

\begin{table}[t]
\centering
\footnotesize

\caption{Empirical Approximation Ratio.}
\label{tab:Approximation-Ratio-implemented-on-Synthetic-Datasets}
\scalebox{0.95}{\footnotesize
\begin{tabular}{ c| c|c c c c}
 \toprule[1pt]
 \#ML/CL points & Algorithm & $k = 5$ & $k = 10$ & $k = 50$ &$k = 100$ \\
\midrule[1pt]
& Alg.~\ref{alg:whole} & \textbf{1.9901} & \textbf{1.9971} & \textbf{1.9964} & \textbf{1.9997} \\
 $1,000$ & Matching & 2.8678 & 2.8295 & 2.9281 & 2.7805 \\ 
 & Greedy & 2.9542 & 2.9793 & 2.9308 & 2.9286 \\ 
\midrule
& Alg.~\ref{alg:whole} & \textbf{1.9892} & \textbf{1.9931} & \textbf{1.9942} & \textbf{1.9986 }\\ 
 $2,000$ & Matching & 2.7537 & 2.8047 & 2.9246 & 2.6884 \\ 
 & Greedy & 2.9992 & 2.9648 & 3.0400 & 2.9073 \\ 
\midrule
& Alg.~\ref{alg:whole} & \textbf{1.9941} & \textbf{1.9908 }&\textbf{ 1.9958} & \textbf{1.9952 }\\ 
 $5,000$ & Matching & 2.5500 & 2.7392 & 3.0239 & 2.8794 \\ 
 & Greedy & 3.0555 & 3.0647 & 3.2250 & 3.1741 \\ 
\midrule
& Alg.~\ref{alg:whole} & \textbf{1.9965} & \textbf{1.9938} & \textbf{1.9979} & \textbf{1.9983} \\ 
 $10,000$ & Matching & 2.5140 & 2.7952 & 3.0236 & 3.1212 \\ 
 & Greedy & 3.3707 & 3.5421 & 3.4695 & 3.4143 \\ 
\bottomrule[1pt]
\end{tabular}
}

\end{table}

In the evaluation, \textbf{approximation ratio} is measured based on the clustering radius ratio obtained on the simulated dataset, i.e., $\text{Approx Ratio} = r_{\text{max}}/r^{*}$, where $r_{\text{max}}$ represents the maximum radius (worst solution cost) obtained from $1,000$ runs of the algorithm ($10$ simulated datasets $\times$ $10$ distinct constrained cases $\times$ $10$ repeated runs) and $r^{*}$ denotes the optimal cost by construction in the simulated dataset. 
Tab.~\ref{tab:Approximation-Ratio-implemented-on-Synthetic-Datasets} presents the empirical ratios, demonstrating that Alg.~\ref{alg:whole} consistently produces better approximation ratios below two, which aligns with the theoretical results presented in Section~\ref{sec:Alg}.

\section{Conclusion}
\label{sec:Conclusion}
In this paper, we confirmed the existence of an efficient approximation algorithm for the constrained $k$-center problem with instance-level ML/CL constraints. Despite the known inapproximability barrier due to the arbitrary CL constraints, we achieved a significant breakthrough by uncovering that the reducible disjoint set structure of CL constraints on $k$-center can lead to constant factor approximation in our analysis. To obtain the best possible 2-approximation, we introduced a structure called \textit{reverse dominating set} (RDS) for identifying the desired set of cluster centers. For efficient RDS computation, we designed an auxiliary graph and proposed two computational methods: the first utilizes a suite of linear programming-based techniques,  grounded in the integral polyhedron of the designated LP; the second leverages a greedy algorithm, predicated on maximum matching within the auxiliary graph. Our work sheds light on devising efficient approximation algorithms to tackle more intricate clustering challenges that incorporate constraints as background knowledge. 
\section*{Limitation and Future Work}
\label{sec:Discussion}
\textbf{Extending to Other Clustering Problems.} Our work introduces a 2-approximation algorithm specifically tailored for constrained $k$-center clustering. However, the current methodology does not directly extend to other fundamental clustering problems, such as $k$-median and $k$-means. We hope to further extend our technique to these clustering problems, potentially improving their accuracy by effectively incorporating relevant background knowledge.

\textbf{Adaptation to Modern Computational Models.} Clustering algorithms today frequently operate within contemporary computational frameworks, including the streaming computational model and Massively Parallel Computation (MPC) models. Extending our approach to these models would substantially enhance clustering accuracy and applicability, particularly in the context of big data applications. Future work will aim to adapt our technique for better performance and scalability within these computational paradigms.

\textbf{Integrating Large Language Models (LLMs) into Constraint Generation.}  Although our algorithms accommodate both disjoint and intersected constraints (CL/ML), the current approach relies on costly procedures, such as Deep Package Inspection (DPI) for traffic classification, to generate these constraints. Future research will explore leveraging LLMs to automatically derive constraints from available background knowledge, significantly reducing the manual effort and associated costs of constraint generation.

\section*{Acknowledgement}

This work is supported by the Taishan Scholars Young Expert Project of Shandong Province (No. tsqn202211215) and the National Natural Science Foundation of China (Nos. 12271098 and 12271259). A preliminary version of this paper has appeared in: 
Efficient Constrained K-center Clustering with Background Knowledge. In Proceedings of AAAI 2024, pp. 20709-20717. 
Part of this work was done when the second author was pursuing her master's degree at Qilu University of Technology. 
We gratefully acknowledge the anonymous referees for their insightful comments and constructive suggestions.

\bibliographystyle{IEEEtran}
\bibliography{tnnls25}

\clearpage
\appendices
\section{Supplementary Material}
\subsection{Comparison on Computational  Complexity } 

For runtime evaluation, we compare both of our algorithms against the traditional $k$-center algorithm, $k$-center algorithm enhanced with the greedy approach~\cite{18_wagstaff2001constrained}, and $k$-center algorithm enhanced with further optimized matching~\cite{21_jia2023efficient}.
The runtimes of the five algorithms are listed in Tab.~\ref{tab:compar_complexity_thm}. Note that while our algorithm has a slightly higher theoretical runtime, the actual runtime gap becomes even smaller when $k$ is relatively small, as is often the case in real-world applications.
\vspace{-6pt}
\begin{table}[th]
\centering
\small
\caption{Comparison of computational complexity.}
\label{tab:compar_complexity_thm}
\scalebox{0.87}{
\begin{tabular}{l|c }
\toprule[1pt]
\textbf{Algorithms} & \textbf{Computational Complexity}\\
\midrule[0.5pt]
 LP-based Alg.~ 3& Polytime \\ 
 BaRDS  Alg.~3 & $O( nk^3\cdot\log n)$ \\ 
 \midrule[0.5pt]
 Matching \cite{21_jia2023efficient} & $O(nk)$ \\ 
 Greedy \cite{18_wagstaff2001constrained} & $O(nk)$ \\ 
 Traditional $k$-center \cite{2_gonzalez1985clustering_minmax}& $O(nk)$ \\ 
\midrule[1pt]
\end{tabular}
}
\end{table}

\begin{table*}[t]

\centering
\caption{Biased sampling constraints}
\scalebox{0.87}{
\footnotesize
\centering
\begin{tabular}{c|c|c c c c|c c c c|c c c c}
\toprule[1pt]
  \multicolumn{2}{c}{Constraints}& \multicolumn{4}{c}{1\%}& \multicolumn{4}{c}{2\%}& \multicolumn{4}{c}{3\%}\\
  \cmidrule(lr){1-2} \cmidrule(lr){3-6} \cmidrule(lr){7-10}\cmidrule(lr){11-14}
  
 \textbf{Datasets} & \textbf{Algorithm} &\textbf{Cost} & \textbf{Purity} & \textbf{NMI} & \textbf{RI} & \textbf{Cost} & \textbf{Purity} & \textbf{NMI} & \textbf{RI} & \textbf{Cost} & \textbf{Purity} & \textbf{NMI} & \textbf{RI}\\
    \midrule[1pt]
    Wine & Greedy & 877.82 & 0.65 & 0.37 & 0.66 &  432.15 & 0.65 & 0.38 & 0.66 & 647.06 & 0.65 & 0.38 & 0.65 \\ 
        & Matching & 651.66 & 0.65 & 0.38 & 0.66 & 410.31 & 0.65 & 0.38 & 0.66 & 587.40 & 0.65 & 0.38 & 0.65 \\
        & \textbf{Alg.3 }& \textbf{642.76} & \textbf{0.69 }& \textbf{0.41} & \textbf{0.71} & \textbf{386.77} & \textbf{0.67} & \textbf{0.41} &	\textbf{0.68} &  \textbf{467.97} & \textbf{0.68} &  \textbf{0.42} & \textbf{0.69} \\
        \midrule
    Cnae-9 & Greedy & 7.27 & 0.13 & 0.03 & 0.15 & 6.45 & 0.12 & 0.02 & 0.14 & 7.49 & 0.14 & 0.06 & 0.19 \\
        & Matching & 6.90 & 0.13 & 0.03 & 0.15  &  6.38 & 0.12 & 0.02 & 0.14 & 7.10 & 0.14 & 0.06 &  0.19  \\ 
        & \textbf{Alg.3} & \textbf{6.66 }& \textbf{0.31} & \textbf{0.25} & \textbf{0.56} & \textbf{6.14 } & \textbf{0.20} & \textbf{0.13} & \textbf{0.36} & \textbf{6.75} & \textbf{0.31} & \textbf{0.24} & \textbf{0.59}\\ 
        \midrule
    Kdd	& Greedy & 3.01 &0.57 & 0.01 & 0.51 & 2.93 & 0.57 & 0.01 & 0.51 & 3.10 & 0.57 & 0.02 & 0.51  \\
        & Matching & 2.89 & 0.57 & 0.01 & 0.51 & 2.86 & 0.57 & 0.01 & 0.51 & 2.98 & 0.58 & 0.02 & 0.51 \\
        & \textbf{Alg.3}  & \textbf{2.55} & \textbf{0.66} & \textbf{0.13} & \textbf{0.57} &  \textbf{2.56} & \textbf{0.64} & \textbf{0.10} & \textbf{0.56} & \textbf{2.54} & \textbf{0.67} & \textbf{0.14} & \textbf{0.58}\\
        \midrule
    Skin & Greedy & 345.17 &  0.80 & 0.08 & 0.56 & 350.77 & 0.80 & 0.08 & 0.56 & 331.17 & 0.80 & 0.09 & 0.56 \\
        & Matching & 329.62 & 0.80 & 0.08 & 0.56  & 334.74 & 0.80 & 0.08 & 0.56  & 324.35 & 0.80 & 0.09 & 0.56 \\
        & \textbf{Alg.3} & \textbf{310.47} & \textbf{0.83} & \textbf{0.25} & \textbf{0.65} & \textbf{322.73} & \textbf{0.83} & \textbf{0.20} & \textbf{0.63} & \textbf{317.41} & \textbf{0.82} & \textbf{0.20} & \textbf{0.61} \\
        \midrule
    Covertype & Greedy & 7217.70 & 0.500 & 0.07& 0.57 & 6995.81 & 0.498 & 0.07 & 0.57 & 7390.57 & 0.498 & 0.07 & 0.57 \\
        & Matching & 7036.07 & 0.499 & 0.07 & 0.57  & \textbf{6949.73} & 0.498 & 0.07 & 0.57 & 7204.55&  0.500 & 0.07 & 0.57 \\
        & \textbf{Alg.3} & \textbf{6747.70} & \textbf{0.504} & \textbf{0.09} & \textbf{0.59}  & 7117.17 & \textbf{0.500} & \textbf{0.09} & \textbf{0.58} & \textbf{6943.59} & \textbf{0.503} & \textbf{0.09 }& \textbf{0.59} \\
\midrule[1pt]
\end{tabular}
\label{tab:biased}
}

\end{table*}

 \begin{table*}[t]
\centering
\caption{The cluster distribution of all the real-word datasets.\\ (We use the \textbf{bold} and \textbf{\textit{italic}} to mark the maximum and minimum ratio of the class respectively.)}
\label{tab:datasets_class}

\scalebox{0.9}{\footnotesize
\begin{tabular}{l|c|c|c|c|c|c|c|c|c|c|c|c|c}
\toprule[1pt]
\textbf{Datasets} & \textbf{Clu.1} & \textbf{Clu.2} & \textbf{Clu.3} & \textbf{Clu.4} & \textbf{Clu.5} & \textbf{Clu.6} & \textbf{Clu.7} & \textbf{Clu.8} & \textbf{Clu.9} & \textbf{Clu.10} & \textbf{Clu.11} & \textbf{Clu.12} & \textbf{Clu.13} \\
\midrule[1pt]
{Wine}  & 33.15\% & \textbf{39.89\%} & \textbf{\textit{26.97\%}} & - & - & - & - & -& - & - & - & - & -\\ 
{Cnae-9 } & 11.11\% & 11.11\% & 11.11\% & 11.11\% & 11.11\% & 11.11\% & 11.11\% & 11.11\% & 11.11\% & - & - & - & -\\ 
{NLS-KDD} \ & \textbf{56.92\%} & \textbf{\textit{43.08\%}} & - & - & - & - & -& - & - & - & - & - & -\\ 
{Skin}  & \textbf{79.25\%} & \textbf{\textit{20.75\%}} & - & - & - & - & -& - & - & - & - & - & -\\ 
{Wide09 } & \textbf{59.70\%} & 27.87\% & 12.30\% & 0.19\% & 0.02\% & 0.0009\% & 0.0007\% & 0.0005\% & 0.0002\% & 0.0002\% & 0.0002\% & 0.0002\% & \textbf{\textit{0.0002\%}}\\
{Covertype} & 36.46\% & \textbf{48.76\% }& 6.15\% & 3.53\% &\textbf{\textit{2.99\%}} & - & - & - & - & -& - & - & -\\
\midrule[1pt]
\end{tabular}
}
\end{table*}

\subsection{Additional Analysis on Experimental Result}
In this subsection, we will provide more analysis of the experimental results for Figs.~\ref{fig:Real-world-Data-Purity-intersected} and ~\ref{fig:Real-world-Data-NMI-intersected}.

 In Fig.~\ref{fig:Real-world-Data-Purity-intersected}, our algorithm clearly outperforms all other methods in subfigures (a) and (b). The results show that our algorithm consistently achieves the best performance on the sparse dataset (Cnae) and the binary classification datasets (KDD and Skin). 

 In subfigures (c) and (d), it also achieves the best performance across all data sets, except for Wide09 and Covertype. These subfigures evaluate clustering quality by varying the repetition rate 
 of constrained points from 0\% to 100\%, using 5\% and 10\% of the data as constrained points. For the two datasets (Wide09 and Covertype), our method is not always the best but remains competitive in clustering quality:  it performs best when the repetition rate of constrained points is low, while the performance degrades when the repetition rate is extremely high (i.e. 100\%, where every sampled point appears in at least two CL sets).
 In general, the performance degradation is primarily due to two reasons: 
 \begin{enumerate}
 
 \item {\bf{Repeated points in CL sets affect the center selection process in our algorithm.}}  Our algorithm is designed for disjoint CL sets (i.e., sets without repeated points) as stated in the algorithmic section. Therefore, it is expected that an extremely high repetition rate of data points in the CL sets negatively impacts the performance. This degradation is observed across nearly all datasets.

\item {\bf{Imbalanced cluster sizes in the ground truth exacerbate the performance drop.}} Both Wide09 and Covertype have imbalanced cluster sizes (i.e. a higher ratio of the standard derivation to mean, as shown in Tab.~\ref{tab:datasets}).  When repeated data points come from smaller clusters, it further weakens the effectiveness of leveraging background knowledge during center selection. This explains why both the matching algorithm and our approach show a similar declining trend in performance in Subfigs.~\ref{subfig:Real-world-Data-Purity-intersected_5} and~\ref{subfig:Real-world-Data-Purity-intersected_10}. In addition, the performance drop of our algorithm is relatively more significant as it relies more heavily on the center selection based on back ground knowledge.
\end{enumerate}

 In Fig.~\ref{fig:Real-world-Data-NMI-intersected}, we evaluate clustering performance using Normalized Mutual Information (NMI). Note that NMI is partially influenced by the distribution of data points across optimal clusters (i.e., the ground truth), that is, whether the number of data points in each cluster is balanced. In particular, for dataset Cnae-9, which has a balanced distribution of data points among clusters, the NMI  score is comparable to other evaluation metrics. However, in cases of imbalanced distributions, there are gaps between NMI and other metrics. Specifically, NMI measures the amount of information shared between the predicted and true distributions, taking into account the overall entropy and how well the distributions match. Notably, if the predicted results split a true category into multiple clusters (even if each resulting cluster is highly pure) the overall shared information may be limited, leading to a lower NMI score. More details on the reasons for the low scores of these indicators in Fig.~\ref{fig:Real-world-Data-NMI-intersected}:
\begin{itemize}
    \item The NMI score is sensitive to imbalanced datasets. 
     We have already discussed this characteristic in Sec.~\ref{sec:Exp} regarding the imbalance in the original datasets. For clarity and more detailed analysis, we present the distribution of each dataset based on the ground truth in Tab.~\ref{tab:datasets_class} below. As shown in the table, most datasets exhibit class imbalance, with the exception of Cnae-9, which has a relatively uniform distribution. Consequently, for Cnae-9, the NMI score is similar to other metrics, particularly the Purity metric. However, all metrics yield low scores on Cnae-9, primarily because the $k$-center algorithm inherently performs poorly when clustering this dataset.
    \item The other metrics demonstrate more consistent performance across various types of datasets, including imbalanced and sparse datasets. Purity and Rand Index (RI) evaluate clustering performance from different perspectives: Purity measures intra-cluster similarity by assessing how homogeneous each cluster is with respect to the ground-truth labels, where higher purity indicates that most points within a cluster belong to the same class. In contrast, Rand Index captures inter-cluster similarity by considering all pairs of samples and measuring the proportion of pairs that are either assigned to the same cluster and share the same label, or assigned to different clusters and have different labels.
\end{itemize}

\begin{table}[ht]
\centering
\footnotesize
\caption{Ablation study regarding the constraints}
\scalebox{1}{
\begin{tabular}{c|c|c c c c}
\toprule[1pt]
Dataset & Constraints & Cost & Purity & NMI & RI\\
 \midrule
Wine & CL & 388.71 & 0.672 & 0.410 & 0.680\\
&ML & \textbf{368.06} & 0.664 & 0.403 & 0.667\\
&ML/CL	& 466.28 & \textbf{0.674} &\textbf{0.411}	&\textbf{0.683} \\
 \midrule
Cnae-9	&CL	&6.88 & 0.33	&0.26	&0.63\\
    & ML & \textbf{6.36 }& 0.12 &0.018	&0.13   \\
    &ML/CL	&6.97 & \textbf{0.35}	&\textbf{0.28}	&\textbf{0.64}\\
 \midrule
Kdd&	CL	&2.572 & 0.66 &0.132	&0.57\\
&ML	&\textbf{2.442} & 0.57 &0.004	&0.51 \\
&ML/CL	&2.570	& \textbf{0.67} &\textbf{0.136} &\textbf{0.58}\\
 \midrule
Skin	&CL	&\textbf{262.17 }& \textbf{0.831} & \textbf{0.24}	&\textbf{0.65}\\
&ML	&297.12	& 0.797 &0.08	&0.56\\
&ML/CL	&306.49 & 0.830 &0.23	&0.64 \\
 \midrule
Wide09	&CL	&1.70	& 0.67 &0.197	&0.58 \\
&ML	&\textbf{1.66}	& 0.673 &0.199	&0.58 \\
&ML/CL	&2.06	& \textbf{0.76} &\textbf{0.308}	&\textbf{0.67} \\
 \midrule
Covertype & CL	& \textbf{5098.16} &0.503& 0.0955 & \textbf{0.590}\\
& ML & 6516.18 &0.497& 0.0801 & 0.523 \\
&ML/CL	&6712.76 & \textbf{0.504} &\textbf{0.0957 }&	0.589\\
\midrule[1pt]

\end{tabular}
}
\label{tab:ablationTable}
\end{table}

\subsection{Ablation Studies}
\label{sec:ablation}
The ablation study of our algorithm in Tab.~\ref{tab:ablationTable} is under different constraints of CL only, ML only, and ML and CL combined using 2\% of the sampled data points, all in non-intersected settings. The constraints are added according to the respective ground truth of each dataset. Overall, the ablation results demonstrate that: i) as the number of constraints increases (e.g., ML \& CL), in general, the cost of solving a constrained $k$-center instance also goes up. We reason that this phenomenon is because incrementally adding constraints is equivalent to reducing the feasible solution space. Therefore, less solution costs are generated by adding only ML or CL; ii) On the other hand, in terms of clustering effectiveness (i.e., NMI, Purity and RI in the table), it is expected that ML/CL mostly outperforms the others as they embed more background knowledge to the clustering.

\subsection{Biased Sampling for ML/CL Sets}
\label{sec:biased_sampling}
In this subsection, we implemented density-biased sampling when generating constraints from our datasets, specifically by sampling an equal number of constraints for each class. The reported results in Tab.~\ref{tab:biased} were attained with the constraint count set at \{1\%, 2\%, 3\%\} of the dataset size, oversampling (and running) 40 times on each dataset. Across all of these datasets, Alg.~\ref{alg:whole} consistently demonstrates the advantages of clustering quality. Note that it is not feasible to include the Wide09 dataset in the biased sampling experiments since one class of the Wide09 dataset contains only one data point. 

Additionally, due to the skewed distribution of each class in the Covertype dataset (Class1: 36.46\%, Class2: 48.76\%, Class3: 6.15\%, Class4: 3.53\%, Class5: 2.99\%, Class6: 1.63\%, Class7: 0.47\%), sampling the same number of constraints from each class creates a substantial disparity between the constraint distribution and the distribution of the raw dataset. Since our work proposes a novel approach to initializing the center set, such a gap between these two distributions exacerbates the quality of the initial center set for our work (relatively more constraints from those classes having fewer data). This explains why our work did not achieve the best performance in terms of cost on the Covertype dataset.

\end{document}